\documentclass{article}

\usepackage{times}

\usepackage{amsmath}
\usepackage{amssymb}
\usepackage{amsthm}
\usepackage{amsfonts}
\usepackage{mathtools}
\usepackage{nicefrac}
\usepackage{xfrac}
\usepackage{bbm}

\usepackage{algorithm}
\usepackage{algorithmic}

\usepackage{enumerate}
\usepackage{enumitem}
\usepackage{paralist}

\usepackage{multibib}
\newcites{appendix}{Appendix References}

\usepackage{multirow}
\usepackage{array}
\usepackage{subfigure} 
\usepackage{graphicx}
\usepackage{float}
\usepackage[usenames, dvipsnames, table]{xcolor}
\usepackage{tabularx}
\usepackage{booktabs}
\usepackage{afterpage}

\usepackage{lipsum}
\usepackage{xspace}
\usepackage{relsize} 
\usepackage{standalone}
\usepackage{flushend} 
\usepackage{etoolbox} 
\PassOptionsToPackage{hyphens}{url}
\usepackage{hyperref}

\usepackage{tikz}
\usetikzlibrary{arrows}
\usetikzlibrary{shapes}
\usetikzlibrary{shapes.misc}
\usetikzlibrary{fit}
\usetikzlibrary{calc}
\usetikzlibrary{intersections}
\usetikzlibrary{positioning}
\usetikzlibrary{decorations.pathmorphing}
\usetikzlibrary{decorations.pathreplacing}
\newtheorem{theorem}{Theorem}
\newtheorem{definition}{Def.}
\newtheorem{lemma}{Lemma}

\makeatletter
\newtheorem*{rep@theorem}{\rep@title}
\newcommand{\newreptheorem}[2]{%
\newenvironment{rep#1}[1]{%
 \def\rep@title{#2 \ref{##1}}%
 \begin{rep@theorem}}%
 {\end{rep@theorem}}}
\makeatother
\newreptheorem{theorem}{Theorem}
\newreptheorem{lemma}{Lemma}

\DeclareMathOperator*{\argmax}{arg\,max}

\def\mynicefrac#1#2{%
    \raise.5ex\hbox{$#1$}%
    \kern-.1em/\kern-.15em%
    \lower.25ex\hbox{$#2$}}

\let\originalleft\left
\let\originalright\right
\renewcommand{\left}{\mathopen{}\mathclose\bgroup\originalleft}
\renewcommand{\right}{\aftergroup\egroup\originalright}

\makeatletter
\newif\if@gather@prefix 
\preto\place@tag@gather{%
  \if@gather@prefix\iftagsleft@ 
    \kern-\gdisplaywidth@ 
    \rlap{\gather@prefix}%
    \kern\gdisplaywidth@ 
  \fi\fi 
} 
\appto\place@tag@gather{%
  \if@gather@prefix\iftagsleft@\else 
    \kern-\displaywidth 
    \rlap{\gather@prefix}%
    \kern\displaywidth 
  \fi\fi 
  \global\@gather@prefixfalse 
} 
\preto\place@tag{%
  \if@gather@prefix\iftagsleft@ 
    \kern-\gdisplaywidth@ 
    \rlap{\gather@prefix}%
    \kern\displaywidth@ 
  \fi\fi 
} 
\appto\place@tag{%
  \if@gather@prefix\iftagsleft@\else 
    \kern-\displaywidth 
    \rlap{\gather@prefix}%
    \kern\displaywidth 
  \fi\fi 
  \global\@gather@prefixfalse 
} 
\newcommand*{\beforetext}[1]{%
  \ifmeasuring@\else
  \gdef\gather@prefix{#1}%
  \global\@gather@prefixtrue 
  \fi
} 
\makeatother

\newcommand{\eqdef}{{\triangleq}}
\newcommand{\transpose}{{\top}} 
\newcommand{\frob}{{\textnormal{\tiny{F}}}}

\newcommand{\trace}{{\textsc{Tr}}}
\newcommand{\timet}{\mathrm{T}}

\newcommand{\algo}{\textnormal{\textrm{SpanCCA}}\xspace}

\usepackage[accepted]{icml2016}
\usepackage[normalem]{ulem}

\icmltitlerunning{Sparse CCA and Beyond}

\begin{document} 
\twocolumn[
\icmltitle{A Simple and Provable Algorithm for Sparse Diagonal CCA
}

\vspace{-3pt}
\icmlauthor{Megasthenis Asteris}{megas@utexas.edu}
\icmlauthor{Anastasios Kyrillidis}{anastasios@utexas.edu}
\icmladdress{Department of Electrical and Computer Engineering, The University of Texas at Austin}
\icmlauthor{Oluwasanmi Koyejo}{sanmi@illinois.edu}
\icmladdress{Stanford University \& University of Illinois at Urbana-Champaign}
\icmlauthor{Russell Poldrack}{poldrack@stanford.edu}
\icmladdress{Department of Psychology,
Stanford University}

\icmlkeywords{Canonical Correlation Analysis, CCA, Sparse, Structure, Graph}

\vskip 0.3in
]

\begin{abstract}
Given two sets of variables, derived from a common set of samples,
sparse Canonical Correlation Analysis (CCA)
seeks linear combinations of a small number of variables in each set,
such that the induced \emph{canonical} variables are maximally correlated.
Sparse CCA is NP-hard.

We propose a novel combinatorial algorithm for sparse diagonal CCA, 
\textit{i.e.}, sparse CCA under the additional assumption that variables within each set are standardized and uncorrelated.
Our algorithm operates on a low rank approximation of the input data and its computational complexity scales linearly with the number of input variables.
It is simple to implement, and parallelizable.
In contrast to most existing approaches, 
our algorithm administers precise control on the sparsity of the extracted canonical vectors,
and comes with theoretical data-dependent global approximation guarantees, that hinge on the spectrum of the input data.
Finally, it can be straightforwardly adapted to other constrained variants of CCA enforcing structure beyond sparsity.

We empirically evaluate the proposed scheme
and apply it on a real neuroimaging dataset to investigate associations between brain activity and behavior measurements.
\end{abstract} 

\section{Introduction}

One of the key objectives in cognitive neuroscience is to localize cognitive processes in the brain, and understand their role in human behavior, as measured by psychological scores and physiological measurements~\cite{posner1988localization}.
This mapping may be investigated by using functional neuroimaging techniques to measure brain activation during carefully designed experimental tasks~\cite{poldrack2006can}.
Following the experimental manipulation, a joint analysis of brain activation and behavioral measurements across subjects can reveal associations that exist between the two~\cite{berman2006studying}.

Similarly, in genetics and molecular biology, several studies involve the joint analysis of multiple assays performed on a single group of patients~\cite{pollack2002microarray,morley2004genetic,stranger2007relative}.
If DNA variants and gene expression measurements are simultaneously available for a set of tissue samples,
a natural objective is to identify correlations between the expression levels of gene subsets and variation in the related genes.

Canonical Correlation Analysis (CCA)~\cite{hotelling1936relations} is a classic method for discovering such linear relationships across two sets of variables
and has been extensively used to investigate associations between multiple views of the same set of observations;
\textit{e.g.}, see~\cite{deleus2011functional,li2012group,smith2015positive} in neuroscience.
Given two datasets $\mathbf{X}$ and $\mathbf{Y}$ 
of dimensions $k \times m$ and $k \times n$, respectively, on $k$ common samples,
CCA seeks linear combinations of the original variables of each type that are maximally correlated.
More formally, the objective is to compute a pair of \emph{canonical vectors (or weights)}~$\mathbf{u}$
and~$\mathbf{v}$
such that the \emph{canonical variables} $\mathbf{X}\mathbf{u}$ and $\mathbf{Y}\mathbf{v}$ achieve the maximum possible correlation:%
\footnote{
  We assume that the variables in $\mathbf{X}$ and $\mathbf{Y}$ are standardized,
  \textit{i.e.}, each column has zero mean and has been scaled to have unit standard deviation.
}
\begin{align}
  \max_{\mathbf{u},~\mathbf{v} \neq 0}
  \frac{
    \mathbf{u}^{\transpose}\mathbf{\Sigma}_{\mathsmaller{\mathrm{XY}}}\mathbf{v}
  }{
    \left( \mathbf{u}^{\transpose}\mathbf{\Sigma}_{\mathsmaller{\mathrm{XX}}}\mathbf{u} \right)^{\sfrac{1}{2}}
    \left( \mathbf{v}^{\transpose}\mathbf{\Sigma}_{\mathsmaller{\mathrm{YY}}}\mathbf{v} \right)^{\sfrac{1}{2}}
  }.
  \label{cca:true}
\end{align}
The optimal canonical pair can be computed via a generalized eigenvalue decomposition involving the empirical estimates of the (cross-) covariance matrices in~\eqref{cca:true}. 

Imaging and behavioral measurements in cognitive neuroscience, similar to genomic data in bioinformatics, 
typically involve hundreds of thousands of variables with only a limited number of samples.
In that case, the CCA objective in~\eqref{cca:true} is ill-posed;
it is always possible to design canonical variables for which the factors in the denominator vanish, irrespective of the data.
Model regularization via constraints such as sparsity,
not only improves the interpretability of the extracted canonical vectors,
but is critical for enabling the recovery of meaningful results.

Sparse CCA seeks to maximize the correlation between subsets of variables of each type, while performing variable selection.
We consider the following sparse diagonal CCA (dCCA) problem, similar to~\cite{witten2009penalized}\footnote{\cite{witten2009penalized} consider a relaxation of~\eqref{cca:def} where the $\ell_{0}$ cardinality constraint on~$\mathbf{u}$ and~$\mathbf{v}$ is replaced by a threshold on the sparsity inducing $\ell_{1}$-norm.
}:
{\setlength\belowdisplayskip{5pt plus 2pt minus 4pt}%
\begin{align}
    \beforetext{\textcolor{gray}{(Sparse dCCA)}}
    \max_{
        \mathbf{u} \in \mathcal{U},
        \mathbf{v} \in \mathcal{V}
        }
    \mathbf{u}^{\transpose}
    \mathbf{\Sigma_{\mathsmaller{\mathrm{X}\mathrm{Y}}}}
    \mathbf{v},
    \label{cca:def}
\end{align}}
{where
\setlength\abovedisplayskip{0pt plus 2pt minus 0pt}%
\begin{equation}
\begin{aligned}
    &\mathcal{U}
    =
    \left\lbrace
        \mathbf{u} \in \mathbb{R}^{m}: \|\mathbf{u}\|_2 = 1, \|
        \mathbf{u}\|_0 \le s_{x}
    \right\rbrace,\\
    &\mathcal{V}
    =
    \left\lbrace
        \mathbf{v} \in \mathbb{R}^{n}: \|\mathbf{v}\|_2 = 1, 
        \|\mathbf{v}\|_0 \le s_{y}
    \right\rbrace,
\end{aligned}
\label{cca:sparse_constraints}
\end{equation}}%
for given parameters $s_{x}$ and $s_{y}$.
The $m \times n$ argument $\mathbf{\Sigma_\mathsmaller{\mathrm{X}\mathrm{Y}}} = \mathbf{X}^{\transpose}\mathbf{Y}$ is the empirical estimation of the cross-covariance matrix between the variables in the two views $\mathbf{X}$ and $\mathbf{Y}$, and it is the input to the optimization problem.
Note that besides the introduced sparsity requirement,~\eqref{cca:def} is obtained from~\eqref{cca:true} treating the covariance matrices~$\mathbf{\Sigma}_{\mathsmaller{\mathrm{XX}}}$ and $\mathbf{\Sigma}_{\mathsmaller{\mathrm{YY}}}$ as identity matrices,
which is common in high dimensions 
~\cite{dudoit2002comparison,tibshirani2003class}.
Equivalently, we implicitly assume that the original variables within each view are standardized and uncorrelated.

The maximization in~\eqref{cca:def} is a sparse Singular Value Decomposition (SVD) problem.
Disregarding the~$\ell_{0}$ cardinality constraint,
although the objective is non-convex, 
the optimal solution of~\eqref{cca:def} can be easily computed, as it coincides with the leading singular vectors of the input matrix $\mathbf{\Sigma_{\mathsmaller{\mathrm{X}\mathrm{Y}}}}$.
The constraint on the number of nonzero entries of $\mathbf{u}$ and $\mathbf{v}$, however, renders the problem NP-hard,
as it can be shown by a reduction to the closely related sparse PCA problem; see Appendix~\ref{sec:hardness}.
Several heuristics have been developed to obtain an approximate solution (see Section \ref{sec:related_work}).

Finally, we note that sparsity alone may be insufficient to obtain interpretable results;
genes participate in groups in biological pathways,
and brain activity tends to be localized forming connected components over an underlying network. 
If higher order structural information is available on a physical system, it is meaningful incorporate that in the optimization~\eqref{cca:def}.
We can then consider \emph{Structured} variants of diagonal CCA~\eqref{cca:def},
by appropriately modifying the feasible regions $\mathcal{U}, \mathcal{V}$ in~\eqref{cca:sparse_constraints} to reflect the desired structure.

\subsection{Our contributions}

We present a novel and efficient combinatorial algorithm for sparse diagonal CCA in \eqref{cca:def}. 
The main idea is to reduce the exponentially large search space of candidate supports of the canonical vectors, by exploring a low-dimensional principal subspace of the input data.
Our algorithm runs in polynomial time --in fact linear-- in the dimension of the input.
It administers precise control over the sparsity of the extracted canonical vectors and can extract components for multiple sparsity values on a single run.
It is simple and trivially parallelizable;
we empirically demonstrate that it achieves an almost linear speedup factor in the number of available processing units.

The algorithm is accompanied with theoretical data-dependent global approximation guarantees with respect to the CCA objective~\eqref{cca:def}; 
this is the first approach with this kind of global guarantees.
The latter depend on the rank $r$ of the low-dimensional space and the spectral decay of the input matrix $\mathbf{\Sigma_{\mathsmaller{\mathrm{X}\mathrm{Y}}}}$.
The main weakness is an exponential dependence of the computational complexity on the accuracy parameter $r$.
In practice, however, disregarding the theoretical approximation guarantees, our algorithm can be executed for any allowable time window.


Finally, we note that our approach is similar to that of~\cite{asteris2014nonnegative} for sparse PCA.
The latter has a similar formulation with~\eqref{cca:def} but is restricted to a positive semidefinite argument~$\mathbf{\Sigma_{\mathsmaller{\mathrm{X}\mathrm{Y}}}}$.
Our main technical contribution is extending those algorithmic ideas and developing theoretical approximation guarantees for the bilinear maximization~\eqref{cca:def}, where the input matrix can be arbitrary.

\subsection{Related Work}
\label{sec:related_work}

Sparse CCA is closely related to sparse PCA;
the latter can be formulated as in~\eqref{cca:def} but the argument $\mathbf{\Sigma_{\mathsmaller{\mathrm{X}\mathrm{Y}}}}$ is replaced by a positive semidefinite matrix.
There is a large volume of work on sparse PCA --see~\cite{zou2006sparse, amini2008high} and references therein--
but these methods cannot be generalized to the CCA problem.
One exception is the work of~\cite{d2007direct} where the authors discuss extensions to the ``non-square case''.
Their approach relies on a semidefinite relaxation.

References to sparsity in CCA date back to~\cite{thorndike1976correlational} and~\cite{thompson1984canonical} who identified the importance of sparsity regularization to obtain meaningful results
However, no specific algorithm was proposed.
Several subsequent works considered a penalized version of the CCA problem in~\eqref{cca:true}, typically under a Langrangian formulation involving a convex relaxation of the $\ell_{0}$ cardinality constraint~\cite{torres2007identifying, hardoon2007sparse, hardoon2011sparse}.
\cite{chu2013sparse} characterize the solutions of the unconstrained problem and formulate convex $\ell_{1}$ minimization problems to seek sparse solutions in that set.
\cite{sriperumbudur2009dc} consider a constrained generalized eigenvalue problem, which partially captures sparse CCA, and frame it as a difference-of-convex functions program.
\cite{wiesel2008greedy} proposed an efficient greedy procedure that gradually expands the supports of the canonical vectors. 
Unlike other methods, this greedy approach allows precise control of the sparsity of the extracted components.


\cite{witten2009penalized,parkhomenko2009sparse} formulate sparse CCA as the optimization~\eqref{cca:def} and in particular considered an $\ell_{1}$ relaxation of the $\ell_{0}$ cardinality constraint.
They suggest an alternating minization approach exploiting the bi-convex nature of the relaxed problem, solving a lasso regression in each step.
The same approach is followed in~\cite{waaijenborg2008quantifying} combining $\ell_2$ and $\ell_1$ regularizers similarly to the elastic net approach for sparse PCA~\cite{zou2005regularization}.
Similar approaches have appeared in the literature for \emph{sparse SVD} \cite{yang2011sparse, lee2010biclustering}.
A common weakness in these approaches is the lack of precise control over sparsity: the mapping between the regularization parameters and the number of nonzero entries in the extracted components is highly nonlinear. 
Further, such methods usually lack provable non-asymptotic approximation guarantees.
Beyond sparsity, \cite{witten2009penalized,witten2009extensions} discuss alternative penalizations such as fused lasso to impose additional structure,
while~\cite{chen2012structured} introduce a group-lasso to promote sparsity with structure.

Finally, although a review of CCA applications 
is beyond the scope of this manuscript,
we simply note that CCA is considered a promising approach for scientific research as evidenced by several recent works in the literature,
\textit{e.g}, in neuroscience
\cite{rustandi2009integrating, deleus2011functional,li2012group,lin2014correspondence,smith2015positive}.

\section{\algo:~An~Algorithm~for~Sparse Diagonal CCA}

We begin this section with a brief discussion of the problem and the key ideas behind our approach. 
Next, we provide an overview of \algo and the accompanying approximation guarantees and conclude with a short analysis.

\subsection{Intuition}
The hardness of the sparse CCA problem \eqref{cca:def} lies in the detection of the optimal supports for the canonical vectors.
In the \emph{unconstrained} problem,
where only a unit $\ell_{2}$-norm constraint is imposed on $\mathbf{u}$ and $\mathbf{v}$, the optimal CCA pair coincides with the top singular vectors of the input argument~$\mathbf{\Sigma_{\mathsmaller{\mathrm{X}\mathrm{Y}}}}$.
In the sparse variant, 
if the optimal supports for $\mathbf{u}$ and $\mathbf{v}$ were known, computing the optimal solution would be straightforward:
the nonzero subvectors of $\mathbf{u}$ and $\mathbf{v}$ would coincide with the leading singular vectors of the $s_{x} \times s_{y}$ submatrix of $\mathbf{\Sigma_{\mathsmaller{\mathrm{X}\mathrm{Y}}}}$, indexed by the two support sets. 
Hence, the bottleneck lies in determining the optimal supports for $\mathbf{u}$ and $\mathbf{v}$.

\textbf{Exhaustive search}\,
A straightforward, brute-force approach is to exhaustively consider all possible supports for $\mathbf{u}$ and $\mathbf{v}$;
for each candidate pair solve the unconstrained CCA problem on the restricted input, and determine the supports for which the objective~\eqref{cca:def} is maximized.
Albeit optimal, this procedure is intractable as the number of candidate supports 
$\binom{m}{s_{x}}\binom{n}{s_{y}}$ is overwhelming even for small values of $s_{x}$ and $s_{y}$. 

\textbf{Thresholding}\, 
On the other hand, a feasible pair of sparse canonical vectors $\mathbf{u}$, $\mathbf{v}$ can be extracted by {hard-thresholding} the solution to the unconstrained problem,
\textit{i.e.},
computing the leading singular vectors $\mathbf{u}$, $\mathbf{v}$ of~$\mathbf{\Sigma_{\mathsmaller{\mathrm{X}\mathrm{Y}}}}$, suppressing to zero all but the $s_{x}$ and $s_{y}$ largest in magnitude entries, respectively, and rescaling to obtain a unit $\ell_{2}$-norm solution.
Essentially, this heuristic resorts to unconstrained CCA for a \emph{guided} selection of the sparse support.

\textbf{Proposed method}\,
Our sparse CCA algorithm covers the ground between these two approaches.
Instead of relying on the solution to the unconstrained problem for the choice of the sparse supports, 
it explores a principal subspace of the input matrix $\mathbf{\Sigma_{\mathsmaller{\mathrm{X}\mathrm{Y}}}}$, spanned by its leading~${r \ge 1}$ singular vector pairs.
For~${r=1}$, its output coincides with that of the thresholding approach,
while for~${r = \min\lbrace{m, n\rbrace}}$ it approximates that of exhaustive search.

Effectively, we solve \eqref{cca:def} on a rank-$r$ approximation of the input~$\mathbf{\Sigma_{\mathsmaller{\mathrm{X}\mathrm{Y}}}}$.
The key observation is that the low inner dimension of the argument matrix can be exploited to substantially reduce the search space: 
our algorithm identifies an (approximately) optimal pair of supports for the low rank sparse CCA problem, without considering the entire collection of possible supports of cardinalities~$s_{x}$ and~$s_{y}$.

\subsection{Overview and Guarantees}
\label{sec:overview_guarantees}
\algo is outlined in Algorithm~\ref{algo:cca}.
The first step is to compute a rank-$r$ approximation $\mathbf{B}$ of the input $\mathbf{\Sigma_{\mathsmaller{\mathrm{X}\mathrm{Y}}}}$, 
where $r$ is an accuracy input parameter,
via the truncated singular value decomposition (SVD) of $\mathbf{\Sigma_{\mathsmaller{\mathrm{X}\mathrm{Y}}}}$.\footnote{
The low-rank approximation can be computed
using faster randomized approaches; see~\cite{halko2011finding}.
Here, for simplicity, we consider the exact case.}
From that point on, the algorithm operates exclusively on~$\mathbf{B}$ effectively solving a low-rank sparse diagonal CCA problem:
\begin{align}
   \max_{
        \mathbf{u} \in \mathcal{U},
        \mathbf{v} \in \mathcal{V}} \quad
   &   
   \mathbf{u}^\transpose \mathbf{B} \mathbf{v},
   \label{cca:on-low-rank}
\end{align}
where~${\mathcal{U} \subseteq \mathbb{S}_{2}^{m-1}}$ and~${\mathcal{V} \subseteq \mathbb{S}_{2}^{n-1}}$ are defined in~\eqref{cca:sparse_constraints}.
As we discuss in the next section, 
we can consider other constrained variants of CCA on potentially arbitrary, non-convex sets.
We do require, however, that there exist procedures to (at least approximately) solve the maximizations
\begin{align}
  & \mathsf{P}_{\mathcal{U}}(\mathbf{a})
   \eqdef
   \argmax_{\mathbf{u} \in \mathcal{U}}\;
   \mathbf{a}^{\transpose}\mathbf{u},
   \label{lowrank:solve-for-x}
   \\
  & \mathsf{P}_{\mathcal{V}}(\mathbf{b})
   \eqdef
   \argmax_{\mathbf{v} \in \mathcal{V}}\;
   \mathbf{b}^{\transpose}\mathbf{v},
   \label{lowrank:solve-for-y}
\end{align}
for any given vectors ${\mathbf{a} \in \mathbb{R}^{m \times 1}}$ and ${\mathbf{b} \in \mathbb{R}^{n \times 1}}$.
Fortunately, this is the case for the sets of sparse unit $\ell_{2}$-norm vectors.
Algorithm~\ref{algo:sparse} outlines an efficient $O(m)$ procedure 
that given $\mathbf{a} \in \mathbb{R}^{m\times 1}$ computes an exact solution to~\eqref{lowrank:solve-for-x} with at most $s \le m$ nonzero entries:
first it determines the~$s$ largest (in magnitude) entries of $\mathbf{a}$ (breaking ties arbitrarily), it zeroes out the remaining entries and re-scales the output to meet the $\ell_{2}$-norm requirement.

The main body of Algorithm~\ref{algo:cca} consists of a single iteration.
In the $i$th round, it independently samples a point, or equivalently direction, $\mathbf{c}_{i}$ from the $r$-dimensional unit $\ell_{2}$ sphere
and uses it to appropriately sample a point $\mathbf{a}_{i}$ in the range of~$\mathbf{B}$.
The latter is then used to compute a feasible solution pair $\mathbf{u}_{i}$, $\mathbf{v}_{i}$ via a two-step procedure:
first the algorithm computes $\mathbf{u}_{i}$ by ``projecting'' $\mathbf{a}_{i}$ onto $\mathcal{U}$ invoking Alg.~\ref{algo:sparse} as a subroutine to solve maximization~\eqref{lowrank:solve-for-x},
and then computes $\mathbf{v}_{i}$ by projecting $\mathbf{b}_{i}=\mathbf{B}^{\transpose}\mathbf{u}_{i}$ onto $\mathcal{V}$ in a similar fashion.
The algorithm repeats this procedure for $\mathrm{T}$ rounds and outputs the pair that achieves the maximum objective value in~\eqref{cca:on-low-rank}.
{We emphasize that consecutive rounds are completely independent and can be executed in parallel.}

\begin{algorithm}[t!]
   \caption{\algo}
   \label{algo:cca}
   {
   \begin{algorithmic}[1]
   \INPUT:
    ${\mathbf{\Sigma_{\mathsmaller{\mathrm{X}\mathrm{Y}}}}}$, a real $m \times n$ matrix.\\
    \quad\,\,\,\,\,$r \in \mathbb{N}_{+}$, the rank of the approximation to be used.\\
    \quad\,\,\,\,\,$\mathrm{T} \in \mathbb{N}_{+}$, the number of samples/iterations.
   \OUTPUT $\mathbf{u}_{\sharp} \in \mathcal{U}$, $\mathbf{v}_{\sharp} \in \mathcal{V}$
   \STATE 
   ${
      \mathbf{U}, \mathbf{\Sigma}, \mathbf{V}
      \leftarrow
      \texttt{SVD}({\mathbf{\Sigma_{\mathsmaller{\mathrm{X}\mathrm{Y}}}}}, r)
   }$
   \hfill \COMMENT{
      ${
         \mathbf{B} \gets \mathbf{U}\mathbf{\Sigma}\mathbf{V}^{\transpose}
      }$
   }
   \FOR{$i=1, \hdots, T$}
      \STATE $\mathbf{c}_{i} \gets \texttt{randn(r)}$
      \hfill \COMMENT{$\sim{}\mathcal{N}(\mathbf{0}, \mathbf{I}_{r \times r})$}
      \STATE $\mathbf{c}_{i} \gets \mathbf{c}_{i} / \|\mathbf{c}_{i}\|_{2}$
      \STATE  $\mathbf{a}_{i} \gets {\mathbf{U}}{\mathbf{\Sigma}}\mathbf{c}_{i}$
      \hfill \COMMENT{$\mathbf{a}_{i} \in \mathbb{R}^{m}$}
      \STATE $\mathbf{u}_{i} \gets \argmax_{\mathbf{u} \in \mathcal{U}} \mathbf{a}_{i}^{\transpose}\mathbf{u}$
      \hfill \COMMENT{$\mathsf{P}_{\mathcal{U}}(\cdot)$}
      \STATE $\mathbf{b}_{i} \gets \mathbf{V}\mathbf{\Sigma}\mathbf{U}^{\transpose}\mathbf{u}_{i}$ 
      \hfill \COMMENT{$\mathbf{b}_{i} \in \mathbb{R}^{n}$}
      \STATE $\mathbf{v}_{i} \gets \argmax_{\mathbf{v} \in \mathcal{V}} \mathbf{b}_{i}^{\transpose}\mathbf{v}$
      \hfill \COMMENT{$\mathsf{P}_{\mathcal{V}}(\cdot)$}
      \STATE $\text{obj}_{i} \gets \mathbf{b}_{i}^{\transpose}\mathbf{v}_{i}$
   \ENDFOR
   \STATE $
        i_{0}
        \gets
        \arg\max_{\substack{  i \in [T] } }
        \text{obj}_{i}$
    \STATE
     $(\mathbf{u}_{\sharp}, \mathbf{v}_{\sharp})
      \gets
      (\mathbf{u}_{i_{0}}, \mathbf{v}_{i_{0}} )$
   \end{algorithmic}
   }
\end{algorithm}
\begin{algorithm}[t!]
    \caption{
        \small{
        $\mathsf{P}_{\mathcal{U}}(\cdot)$ for 
        $\mathcal{U}
        \eqdef
        \bigl\lbrace \mathbf{u} \in \mathbb{S}_{2}^{m-1}: \|\mathbf{u}\|_{0} \le s\bigr\rbrace
        $}
    }
    \label{algo:sparse}
    \begin{algorithmic}[1]
    \INPUT:
        $\mathbf{a} \in \mathbb{R}^{d \times 1}$.
    \OUTPUT
        $
            \mathbf{u}_{0} = \argmax_{\mathbf{u} \in \mathcal{U}} \mathbf{a}^{\transpose}\mathbf{u}$
   \STATE $\mathbf{u}_{0} \gets \mathbf{0}_{d \times 1}$
   \STATE $t \gets $ index of $s$th order element of $\texttt{abs}(\mathbf{a})$
   \STATE $\mathcal{I} \gets \lbrace i: \lvert a_{i} \rvert \ge \lvert a_{t} \rvert\rbrace$
   \STATE $\mathbf{u}_{0}[i] \gets \mathbf{a}[i], \forall i \in \mathcal{I}$
   \STATE $\mathbf{u}_{0} \gets \mathbf{u}_{0} / \|\mathbf{u}_{0}\|_{2}$
   \end{algorithmic}
\end{algorithm}

For a sufficiently large number~$\mathrm{T}$ of rounds (or samples)
the procedure guarantees that the output pair will be approximately optimal in terms of the objective for the low-rank problem~\eqref{cca:on-low-rank}.
That, it turn, translates to approximation guarantees for the full-rank sparse CCA problem~\eqref{cca:def}:
\begin{theorem}
\label{thm:main-algo-guarantees}
For any real ${m \times n}$ matrix~$\mathbf{\Sigma_{\mathsmaller{\mathrm{X}\mathrm{Y}}}}$,
$\epsilon \in (0, 1)$, and $r \le \max\lbrace m, n\rbrace$,
Algorithm~\ref{algo:cca} with input $\mathbf{\Sigma_{\mathsmaller{\mathrm{X}\mathrm{Y}}}}$, $r$,
and $\mathrm{T}=\widetilde{O}\bigl(2^{r \cdot \log_{2}(\sfrac{2}{\epsilon})} \bigr)$
outputs $\mathbf{u}_{\sharp} \in \mathcal{U}$
and $\mathbf{v}_{\sharp} \in \mathcal{V}$ such that
   \begin{align}
     \mathbf{u}_{\sharp}^{\transpose}
     \mathbf{\Sigma_{\mathsmaller{\mathrm{X}\mathrm{Y}}}}
     \mathbf{v}_{\sharp}
     \ge
     \mathbf{u}_{\star}^{\transpose} \mathbf{\Sigma_{\mathsmaller{\mathrm{X}\mathrm{Y}}}} \mathbf{v}_{\star}
     - 
     \epsilon \cdot \sigma_{1}(\mathbf{\Sigma_{\mathsmaller{\mathrm{X}\mathrm{Y}}}})
     -
     2 \sigma_{r+1}(\mathbf{\Sigma_{\mathsmaller{\mathrm{X}\mathrm{Y}}}}),
     \nonumber
   \end{align}
   in time 
   $
   \timet_{\mathsmaller{\mathsf{SVD}}}(r) + O\mathopen{}\bigl(\mathrm{T} \cdot \bigl(\timet_{\mathcal{U}}+\timet_{\mathcal{V}}+r \cdot \max\lbrace m, n\rbrace \bigr)\bigr)
   $.
\end{theorem}
Here, 
$\mathbf{u}_{\star}$ and $\mathbf{v}_{\star}$ denote the unknown optimal pair of canonical vectors satisfying the desired constraints. 
$\timet_{\mathsmaller{\mathsf{SVD}}}(r)$ denotes the time to compute the rank-$r$ truncated SVD of the input~$\mathbf{\Sigma_{\mathsmaller{\mathrm{X}\mathrm{Y}}}}$, while $\timet_{\mathcal{U}}$ and $\timet_{\mathcal{V}}$ denote the time required to compute the maximizations~\eqref{lowrank:solve-for-x} and~\eqref{lowrank:solve-for-y}, respectively,
which in the case of Alg.~\ref{algo:sparse} are linear in the dimensions $m$ and $n$.


The first term in the additive error is due to the sampling approach of Alg.~\ref{algo:cca}.
The second term is due to the fact that the algorithm operates on the rank-$r$ surrogate matrix~$\mathbf{B}$.
Theorem~\ref{thm:main-algo-guarantees} 
establishes a trade-off between the computational complexity of Alg.~\ref{algo:cca} and the quality of the approximation guarantees: the latter improves as $r$ increases, but the former depends exponentially in~$r$.

Finally,
in the special case where we impose sparsity constraints on only one of the two variables, say $\mathbf{u}$, while allowing the second variable, here $\mathbf{v}$, to be any vector with unit $\ell_{2}$ norm, 
we obtain stronger guarantees.
\begin{theorem}
    \label{thm:main-algo-guarantees-special-case}
    If $\mathcal{V}=\lbrace \mathbf{v}: \|\mathbf{v}\|_{2}=1\rbrace$,
    \textit{i.e.}, if no constraint is imposed on variable $\mathbf{v}$ besides unit length,
    then Algorithm~\ref{algo:cca} under the same configuration as that in Theorem~\ref{thm:main-algo-guarantees} 
    outputs $\mathbf{u}_{\sharp} \in \mathcal{U}$
and $\mathbf{v}_{\sharp} \in \mathcal{V}$ such that
   \begin{align}
     \mathbf{u}_{\sharp}^{\transpose}
     \mathbf{\Sigma_{\mathsmaller{\mathrm{X}\mathrm{Y}}}}
     \mathbf{v}_{\sharp}
     \ge
     (1-\epsilon) \cdot
     \mathbf{u}_{\star}^{\transpose} \mathbf{\Sigma_{\mathsmaller{\mathrm{X}\mathrm{Y}}}} \mathbf{v}_{\star}
     -
     2 \cdot \sigma_{r+1}(\mathbf{\Sigma_{\mathsmaller{\mathrm{X}\mathrm{Y}}}}).
     \nonumber
   \end{align}
\end{theorem}
Theorem~\ref{thm:main-algo-guarantees-special-case} implies that due to the flexibility in the choice of the canonical vector $\mathbf{v}$,
Alg.~\ref{algo:cca} solves the low-rank problem~\eqref{cca:on-low-rank} within a multiplicative $(1-\epsilon)$-factor from the optimal; 
the extra additive error term is once again due to the fact that the algorithm operates on the rank-$r$ approximation $\mathbf{B}$ instead of the original input~$\mathbf{\Sigma_{\mathsmaller{\mathrm{X}\mathrm{Y}}}}$.
In this case, the optimal choice of $\mathbf{v}$ in~\eqref{lowrank:solve-for-y} is just a scaled version of the argument $\mathbf{b}$.
Finally, we note that if constraints need to be applied on $\mathbf{v}$ instead of $\mathbf{u}$, then the same guarantees can be obtained by applying Algorithm~\ref{algo:cca} on $\mathbf{\Sigma_{\mathsmaller{\mathrm{X}\mathrm{Y}}}}^{\transpose}$.
A formal proof for Theorem~\ref{thm:main-algo-guarantees-special-case} is provided in the Appendix, Sec.~\ref{sec:approx-proof}.

Overall, \algo is simple to implement and is trivially parallelizable: the main iteration can be split across and arbitrary number of processing units achieving a potentially linear speedup.
It is the first algorithm for sparse diagonal CCA with data-dependent global approximation guarantees.
As discussed in Theorem~\ref{thm:main-algo-guarantees}, the input accuracy parameter $r$ establishes a trade-off between the running time and the tightness of the theoretical guarantees.
Its complexity scales linearly in the dimensions of the input for any constant $r$,
but admittedly becomes prohibitive even for moderate values of $r$.
In practice, if the spectrum of the data exhibits sharp decay, we may be able to obtain useful approximation guarantees even for small values of $r$ such as $2$ or $3$.
Moreover, disregarding the theoretical guarantees, the algorithm can always be executed for any rank $r$ and an arbitrary number of iterations $\mathrm{T}$.
In Section \ref{sec:experiments}, we empirically show that moderate values for $r$ and $\mathrm{T}$ can potentially achieve better solutions compared to state of the art.
We note, however, that the tuning of those  parameters needs to be investigated.


\subsection{Analysis}
Let $\mathbf{B} = \mathbf{U}\mathbf{\Sigma}\mathbf{V}^{\transpose}$,
and $\mathbf{u}_{\mathsmaller{\mathtt{(B)}}}, \mathbf{v}_{\mathsmaller{\mathtt{(B)}}}$ be a pair that  maximizes --not necessarily uniquely-- the objective $\mathbf{u}^{\transpose}\mathbf{B}\mathbf{v}$ in~\eqref{cca:on-low-rank} over all feasible solutions.
We assume that $\mathbf{u}_{\mathsmaller{\mathtt{(B)}}}^{\transpose}\mathbf{B}\mathbf{v}_{\mathsmaller{\mathtt{(B)}}} > 0$.\footnote{Observe that this is always true for any nonzero argument $\mathbf{B}$ as long as at least one of the two variables $\mathbf{u}$ and $\mathbf{v}$ can take arbitrary signs.
It is hence true under vanilla sparsity constraints.}
Define the $r \times 1$ vector $\mathbf{c}_{\mathsmaller{\mathtt{(B)}}}\eqdef\mathbf{V}^{\transpose}\mathbf{v}_{\mathsmaller{\mathtt{(B)}}}$ and let $\rho$ denote its $\ell_{2}$ norm. 
Then, $0<\rho\le 1$;
the upper bound follows from the fact that the $r$ columns of $\mathbf{V}$ are orthonormal and ${\|\mathbf{v}_{\mathsmaller{\mathtt{(B)}}}\|_{2} = 1}$, while the lower follows by the aforementioned assumption. 
Finally, define $\overline{\mathbf{c}}_{\mathsmaller{\mathtt{(B)}}} = \mathbf{c}_{\mathsmaller{\mathtt{(B)}}} / \rho$, the projection of $\mathbf{c}_{\mathsmaller{\mathtt{(B)}}}$ on the unit $\ell_{2}$-sphere~$\mathbb{S}_{2}^{r-1}$.

\begin{definition}
For any $\epsilon \in (0,1)$, an $\epsilon$-net of $\mathbb{S}_{2}^{r-1}$ is a finite collection $N$ of points in $\mathbb{R}^{}$ such that for any $\mathbf{c} \in \mathbb{S}_{2}^{r-1}$, $N$ contains a point $\mathbf{c}\prime$ such that $\|\mathbf{c}\prime-\mathbf{c}\| \le \epsilon$.
\end{definition}
\begin{lemma}[\cite{vershynin2010introduction}, Lemma 5.2]
    \label{lem:vershynin12}
    For any $\epsilon \in (0, 1)$,
    there exists an $\epsilon$-net of $\mathbb{S}_{2}^{r-1}$ equipped with the Euclidean metric, with at most 
    $\left(1+\sfrac{2}{\epsilon}\right)^r$ points. 
\end{lemma}

Algorithm~\ref{algo:cca} runs in an iteration with $\mathrm{T}$ rounds.
In each round, it independently samples a point $\mathbf{c}_{i}$ from~$\mathbb{S}_{2}^{r-1}$, 
by randomly generating a vector according to a spherical Gaussian distribution and appropriately scaling its length.
Based on Lemma~\ref{lem:vershynin12} and elementary counting arguments, for sufficiently large $\mathrm{T}$ the collection of sampled points forms a $\epsilon$-net of $\mathbb{S}_{2}^{r-1}$ with high probability:
\begin{lemma} 
  \label{lem:num-of-random-points}
  For any $\epsilon, \delta \in (0, 1)$,
  a set of $\mathrm{T}={O\mathopen{}\bigl( r (\sfrac{\epsilon}{4})^{-r} \cdot \ln{\sfrac{4}{{\epsilon}\cdot{\delta}}} \bigr)}$
  randomly and independently drawn points uniformly distributed on~$\mathbb{S}_{2}^{r-1}$
  suffices to construct an $\sfrac{\epsilon}{2}$-net of~$\mathbb{S}_{2}^{r-1}$ with probability at least $1 - \delta$.
\end{lemma}
It follows that there exists $i_{\star} \in [\mathrm{T}]$ such that
\begin{align}
  \label{eq:net_property_main}
  \|\mathbf{c}_{i_{\star}} - \overline{\mathbf{c}}_{\mathsmaller{\mathtt{(B)}}}\|_{2}
  \le
  \epsilon/2.
\end{align}

In the $i_{\star}$th round, the algorithm samples the point $\mathbf{c}_{i_{\star}}$ and computes a feasible pair $(\mathbf{u}_{i_{\star}}, \mathbf{v}_{i_{\star}})$ via the two step maximization procedure, 
that is, 
 \begin{align}
    \mathbf{u}_{i_{\star}}
    \eqdef
    \argmax_{\mathbf{u}\in \mathcal{U}}
       \mathbf{u}^{\transpose}
       \mathbf{U}\mathbf{\Sigma}
       \mathbf{c}_{i_{\star}}
       \; \text{and} \;
             \mathbf{v}_{i_{\star}}
      \eqdef
      \argmax_{\mathbf{v}\in \mathcal{V}}
         \mathbf{u}_{i_{\star}}^{\transpose}
         \mathbf{\Sigma_{\mathsmaller{\mathrm{X}\mathrm{Y}}}}
         \mathbf{v}.
         \nonumber
\end{align}
We have:
   \begin{align}
      {\mathbf{u}}_{\mathsmaller{\mathtt{(B)}}}^{\transpose}
      \mathbf{B}
      {\mathbf{v}}_{\mathsmaller{\mathtt{(B)}}}
      &=
         \rho \cdot 
         \mathbf{u}_{\mathsmaller{\mathtt{(B)}}}^{\transpose}
         \mathbf{U}\mathbf{\Sigma}
         \overline{\mathbf{c}}_{\mathsmaller{\mathtt{(B)}}}
      \nonumber\\&=
      \rho \cdot 
         \mathbf{u}_{\mathsmaller{\mathtt{(B)}}}^{\transpose}
         \mathbf{U}\mathbf{\Sigma}
         \mathbf{c}_{i_{\star}}
      +
         \rho \cdot 
         \mathbf{u}_{\mathsmaller{\mathtt{(B)}}}^{\transpose}
         \mathbf{U}\mathbf{\Sigma}
         \bigl(\overline{\mathbf{c}}_{\mathsmaller{\mathtt{(B)}}}-\mathbf{c}_{i_{\star}}\bigr)
      \nonumber\\&\le
         \rho \cdot 
         \mathbf{u}_{i_{\star}}^{\transpose}
         \mathbf{U}\mathbf{\Sigma}
         \mathbf{c}_{i_{\star}}
      +
         \rho \cdot 
         \mathbf{u}_{\mathsmaller{\mathtt{(B)}}}^{\transpose}
         \mathbf{U}\mathbf{\Sigma}
         \bigl(\overline{\mathbf{c}}_{\mathsmaller{\mathtt{(B)}}}-\mathbf{c}_{i_{\star}}\bigr)
      \nonumber\\&\le
         \rho \cdot 
         \mathbf{u}_{i_{\star}}^{\transpose}
         \mathbf{U}\mathbf{\Sigma}
         \mathbf{c}_{i_{\star}}
      +
         \tfrac{\epsilon}{2} \cdot \sigma_{1}(\mathbf{\Sigma_{\mathsmaller{\mathrm{X}\mathrm{Y}}}})
      .
      \label{base-inequality_in_main}
   \end{align}
The first step follows by the definition of $\overline{\mathbf{c}}_{\mathsmaller{\mathtt{(B)}}}$
and the second by linearity. 
The first inequality follows from the fact that $\mathbf{u}_{i_{\star}}$  maximizes the first term over all $\mathbf{u}\in \mathcal{U}$.
   The last inequality follows straightforwardly from the fact
   that $\|\mathbf{u}_{\mathsmaller{\mathtt{(B)}}}\|_{2}=1$ and $\rho \le 1$ (see Lemma~\ref{lemma:abs-trace-XAY-ub}).
   Using similar arguments, 
   \begin{align}
        \rho \cdot 
         \mathbf{u}_{i_{\star}}^{\transpose}
         \mathbf{U}\mathbf{\Sigma}
         \mathbf{c}_{i_{\star}}
&=
         \mathbf{u}_{i_{\star}}^{\transpose}
         \mathbf{U}\mathbf{\Sigma}
         \mathbf{V}^{\transpose}
         \mathbf{v}_{\mathsmaller{\mathtt{(B)}}}
         +
         \rho \cdot 
         \mathbf{u}_{i_{\star}}^{\transpose}
         \mathbf{U}\mathbf{\Sigma}
         \left(
            \mathbf{c}_{i_{\star}} - \overline{\mathbf{c}}_{\mathsmaller{\mathtt{(B)}}}
         \right)
      \nonumber\\&\le
         \mathbf{u}_{i_{\star}}^{\transpose}
         \mathbf{B}
         \mathbf{v}_{i_{\star}}
         +
         \tfrac{\epsilon}{2} \cdot \sigma_{1}(\mathbf{\Sigma_{\mathsmaller{\mathrm{X}\mathrm{Y}}}}).
         \label{bound_on_sharp_x_in_main}
   \end{align}
   The inequality follows by the fact that $\mathbf{v}_{i_{\star}}$ 
   maximizes the inner product with $\mathbf{B}^{\transpose}\mathbf{u}_{i_{\star}}$ over all ${\mathbf{v} \in \mathcal{V}}$,
   as well as that $\|\mathbf{u}_{\mathsmaller{\mathtt{(B)}}}\|_{2}=1$ and $\rho \le 1$.
   Combining~\eqref{base-inequality_in_main} and~\eqref{bound_on_sharp_x_in_main}, we obtain
   \begin{align}
         \mathbf{u}_{i_{\star}}^{\transpose}
         \mathbf{B}
         \mathbf{v}_{i_{\star}}
      \ge
         \mathbf{u}_{\mathsmaller{\mathtt{(B)}}}^{\transpose}
         \mathbf{B}
         \mathbf{v}_{\mathsmaller{\mathtt{(B)}}}
      -
      \epsilon
      \cdot \sigma_{1}(\mathbf{\Sigma_{\mathsmaller{\mathrm{X}\mathrm{Y}}}}).
      \label{eq:guarantees_for_low_rank}
   \end{align}
   Algorithm~\ref{algo:cca} computes multiple candidate solution pairs and outputs the pair $(\mathbf{u}_{\sharp}, \mathbf{v}_{\sharp})$ that achieves the maximum objective value.
   The latter is at least as high as that achieved by 
   $(\mathbf{u}_{i_{\star}}, \mathbf{v}_{i_{\star}})$.

Inequality~\eqref{eq:guarantees_for_low_rank} establishes an approximation guarantee for the low-rank problem~\eqref{cca:on-low-rank}.
Those can be translated to guarantees on the original problem with input argument~$\mathbf{\Sigma_{\mathsmaller{\mathrm{X}\mathrm{Y}}}}$.
Let $\mathbf{u}_{\star}$ and $\mathbf{v}_{\star}$ denote the (unknown) optimal solution of the sparse CCA problem \eqref{cca:def}.
By the definition of $\mathbf{u}_{\mathsmaller{\mathtt{(B)}}}$ and $\mathbf{v}_{\mathsmaller{\mathtt{(B)}}}$, it follows that
\begin{align}
   \mathbf{u}_{\mathsmaller{\mathtt{(B)}}}^{\transpose}\mathbf{B}\mathbf{v}_{\mathsmaller{\mathtt{(B)}}}
   \ge
   \mathbf{u}_{\star}^{\transpose}
   \mathbf{B}
   \mathbf{v}_{\star}
   &=
   \mathbf{u}_{\star}^{\transpose}
   \mathbf{\Sigma_{\mathsmaller{\mathrm{X}\mathrm{Y}}}}
   \mathbf{v}_{\star}
   -
   \mathbf{u}_{\star}^{\transpose}
   (\mathbf{\Sigma_{\mathsmaller{\mathrm{X}\mathrm{Y}}}}-\mathbf{B})
   \mathbf{v}_{\star}
   \nonumber\\& \ge
   \mathbf{u}_{\star}^{\transpose}
   \mathbf{\Sigma_{\mathsmaller{\mathrm{X}\mathrm{Y}}}}
   \mathbf{v}_{\star}
   -
   \sigma_{r+1}(\mathbf{\Sigma_{\mathsmaller{\mathrm{X}\mathrm{Y}}}}).
   \label{eq:bp03_main}
\end{align}
Similarly, 
\begin{align}
   \mathbf{u}_{\sharp}^{\transpose}
   \mathbf{\Sigma_{\mathsmaller{\mathrm{X}\mathrm{Y}}}}
   \mathbf{v}_{\sharp}
&=
   \mathbf{u}_{\sharp}^{\transpose}
   \mathbf{B}
   \mathbf{v}_{\sharp}
   -
   \mathbf{u}_{\sharp}^{\transpose}
   (\mathbf{B}-\mathbf{\Sigma_{\mathsmaller{\mathrm{X}\mathrm{Y}}}})
   \mathbf{v}_{\sharp}
   \nonumber\\ 
&\ge
   \mathbf{u}_{\sharp}^{\transpose}
   \mathbf{B}
   \mathbf{v}_{\sharp}
   -
   \sigma_{r+1}(\mathbf{\Sigma_{\mathsmaller{\mathrm{X}\mathrm{Y}}}}).
   \label{eq:bp02_main}
\end{align}
Combining~\eqref{eq:bp03_main} and~\eqref{eq:bp02_main}
with~\eqref{eq:guarantees_for_low_rank}, 
we obtain the approximation guarantees of Theorem~\ref{thm:main-algo-guarantees}.

The running time of Algorithm~\ref{algo:cca} follows straightforwardly by inspection.
The algorithm first computes the truncated singular value decomposition of inner dimension $r$ in time denoted by $\timet_{\texttt{SVD}}(r)$.
Subsequently, it performs $\mathrm{T}$ iterations.
The cost of each iteration is determined by the cost of the matrix-vector multiplications and the running times $\timet_{\mathcal{U}}$ and $\timet_{\mathcal{V}}$ of the operators $\mathsf{P}_{\mathcal{U}}(\cdot)$ and $\mathsf{P}_{\mathcal{V}}(\cdot)$. 
Note that matrix multiplications can exploit the available matrix decomposition and are performed in time ${r \cdot \max\lbrace m, n\rbrace}$.
Substituting the value of $\mathrm{T}$ with that specified in Lemma~\ref{lem:num-of-random-points} completes the proof of Thm.~\ref{thm:main-algo-guarantees}.
The proof of Theorem~\ref{thm:main-algo-guarantees-special-case} follows a similar path; see Appendix Sec.~\ref{sec:approx-proof}.

\section{Beyond Sparsity: Structured CCA}

While enforcing sparsity results in succinct models,
the latter may fall short in capturing the true interactions in a physical system, especially when the number of samples is limited.
Incorporating additional prior structural information can improve interpretability\footnote{This is a shared insight in the broader area of sparse approximations~\cite{baraniuk2010model, huang2011learning, bach2012structured, kyrillidis2012combinatorial}.};
\textit{e.g.},
\cite{du2014novel} argue that a structure-aware sparse CCA  incorporating group-like structure obtains biologically more meaningful results,
while \cite{lin2014correspondence} demonstrated that group prior knowledge improved performance compared to standard sparse CCA in a task of identifying brain regions susceptible to schizophrenia.
Several works suggest using structure-inducing regularizers to promote smoothness 
\cite{witten2009penalized,chen2013structure, kobayashi2014s3cca} 
or group sparse structure~\cite{chu2013sparse} in CCA.

Our sparse CCA algorithm and its theoretical approximation guarantees in Theorems~\ref{thm:main-algo-guarantees} and~\ref{thm:main-algo-guarantees-special-case}
extend straightforwardly to constraints beyond sparsity.
The only assumption on the feasible sets $\mathcal{U}$ and $\mathcal{V}$
is that there exist tractable procedures $\mathsf{P}_{\mathcal{U}}$ and $\mathsf{P}_{\mathcal{V}}$ that solve the constrained maximizations~\eqref{lowrank:solve-for-x} and ~\eqref{lowrank:solve-for-y}, respectively.
The specific structure of the feasible sets only manifests itself through these subroutines, 
\textit{e.g.}, Alg.~\ref{algo:sparse} for the case of sparsity constraints.
Therefore, Alg.~\ref{algo:cca} can be straightforwardly adapted to any structural constraint for which the aforementioned conditions are satisfied.

In fact, observe that under the unit $\ell_{2}$ restriction on the feasible vectors, 
the maximizations in~\eqref{lowrank:solve-for-x} and~\eqref{lowrank:solve-for-y} are equivalent to computing the Euclidean projection of a given real vector on the (nonconvex) sets $\mathcal{U}$ and $\mathcal{V}$.
Such exact or approximate projection procedures exist for several interesting constraints beyond sparsity such as smooth or group sparsity \cite{huang2011learning, baldassarre2013group, kyrillidis2015structured}, sparsity constraints onto norm balls \cite{kyrillidis2012hard}, or even sparsity patterns guided by underlying graphs~\cite{hegde2015nearly,asteris2015pathpca}. 
\section{Experiments}
\label{sec:experiments}

We empirically evaluate our algorithm on two real datasets:
\begin{inparaenum}[\itshape i)\upshape]
\item a publicly available breast cancer dataset~\cite{chin2006genomic}, also used in the evaluation of~\cite{witten2009penalized}, and
\item a neuroimaging dataset obtained from the Human Connectome Project~\cite{van2013wu} on which we investigate associations between brain activation and behavior measurements.
\end{inparaenum}


\subsection{Breast Cancer Dataset}

The breast cancer dataset~\cite{chin2006genomic} consists of gene expression and DNA copy number measurements on a set of $89$ tissue samples.
Among others, it contains 
a $89 \times 2149$ matrix (DNA) with CGH spots for each sample
and a $89 \times 19672$ matrix (RNA) of genes,
along with information for the chromosomal locations of each CGH spot and each gene.
As described in~\cite{witten2009penalized},
this dataset can be used to perform integrative analysis of gene expression and DNA copy number data,
and in particular to identify sets of genes that have expression that is correlated with a set of chromosomal gains or losses.

\begin{figure}[t!]
\centering
\includegraphics[width=0.99\linewidth]{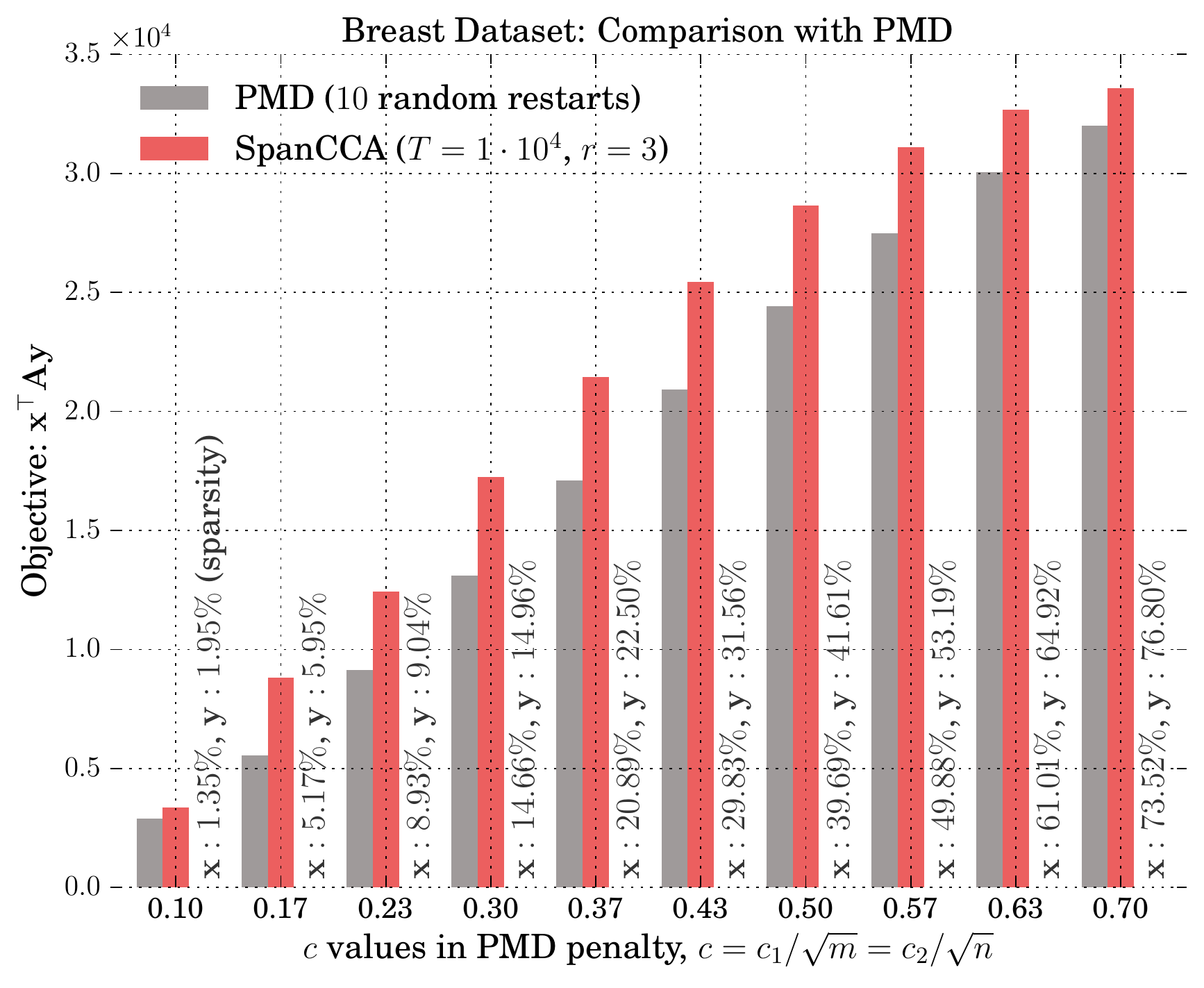}
\vspace{1em}

{\small
\rowcolors{2}{white}{black!05!white}
\begin{tabular}{lccc}
  \toprule
  & Avg Exec. Time & Configuration \\ 
  \cmidrule{1-1} \cmidrule{2-2} \cmidrule{3-3} \cmidrule{4-4}
   PMD & $\sim 44$ seconds  & $10$ rand. restarts\\
   \algo & $\sim 24$ seconds & $T=10^4$, $r=3$.\\
  \bottomrule
\end{tabular}
}
\caption{%
Comparison of \algo and the PMD algorithm~\cite{witten2009penalized}.
We configure PMD with $\ell_{1}$-norm thresholds $c_{1}=c \cdot \sqrt{m}$ and $c_{2}=c \cdot \sqrt{n}$,
and consider various values of the constant $c \in (0,1)$.
For each~$c$, we run PMD $10$ times, select the canonical vectors $\mathbf{x}$, $\mathbf{y}$ that achieve the highest objective value and count their nonzero entries (depicted as percentage of the corresponding dimension).
Finally, we run our \algo algorithm with ${T=10^4}$ and ${r=3}$, using the latter as target sparsities,
and compare the objective values achieved by the two methods.
Execution times remain approximately the same for all target sparsity values (equiv. all~$c$).
}
\label{fig:breast_comparison}
\end{figure}
We run our algorithm on the breast cancer dataset  and compare the output with the PMD algorithm of~\cite{witten2009penalized};
PMD is regarded as state of the art by practitioners and has been used --in its original form or slightly modified-- in several neuroscience and biomedical applications; see also Section \ref{sec:related_work}.
The input to both algorithms is the $m \times n$ matrix $\mathbf{\Sigma_{\mathsmaller{\mathrm{X}\mathrm{Y}}}}=\mathbf{X}^{\transpose}\mathbf{Y}$ (${m=2149}$, ${n=19672}$),
where $\mathbf{X}$ and $\mathbf{Y}$ are obtained from the aforementioned DNA and RNA matrices upon feature standardization.
Recall that PMD is an iterative, alternating optimization scheme,
where the sparsity of the extracted components $\mathbf{x}$ and $\mathbf{y}$ is implicitly controlled by enforcing upper bounds $c_{1}$ and $c_{2}$ on their $\ell_{1}$ norm, respectively, with $1\le c_{1} \le \sqrt{m}$ and $ 1\le c_{2} \le \sqrt{n}$.
Here, for simplicity, we set ${c_{1} = c \sqrt{m}}$ and ${c_{2} = c \sqrt{n}}$ and consider multiple values of the constant $c$ in $(0, 1)$.
Note that under this configuration, for any given value of~$c$, we expect that the extracted components will be approximately equally sparse, relatively to their dimension.

For each~$c$, we first run the PMD algorithm $10$ times with random initializations,
determine the pair of components $\mathbf{x}$ and $\mathbf{y}$ that achieves the highest objective value, and count the number of nonzero entries of both components as a percentage of their corresponding dimension.
Subsequently, we run \algo (Alg.~\ref{algo:cca}) with parameters ${T=10^4}$, ${r=3}$, and target sparsity equal to that of the former PMD output.
Recall that our algorithm administers precise control on the number of nonzero entries of the extracted components.

Figure~\ref{fig:breast_comparison} depicts the objective value achieved by the two algorithms, 
as well as the corresponding sparsity level of the extracted components.
SpanCCA achieves a higher objective value in all cases.
Finally, note that under the above configuration, 
both algorithms run for a few seconds per target sparsity, 
with \algo running approximately half the time of PMD.

\subsection{Brain Imaging Dataset}
\label{sec:brain_imaging_dataset}
We analyzed functional statistical maps and behavioral variables from $497$ subjects available from the Human Connectome Project (HCP)~\cite{van2013wu}. 
The HCP consists of high-quality imaging and behavioral data, collected from a large sample of healthy adult subjects, motivated by the goal of advancing knowledge between human brain function and its association to behavior. 
We apply our algorithm to investigate the shared co-variation between patterns of brain activity as measured by the experimental tasks, and behavioral variables. 
We selected the same subset of behavioral variables examined by~\cite{smith2015positive}, which include scores from psychological tests, physiological measurements, and self reported behavior questionnaires
($\mathbf{Y}$ dataset with dimensions $497 \times 38$).

For each subject, we collected statistical maps corresponding to ``$n$-back'' task.
These statistical maps summarize the activation of each voxel in response to the experimental manipulation. 
In the ``$n$-back'' task, designed to measure working memory, items are presented one at a time and subjects identify each item that repeats relative to the item that occurred $n$ items before.
Further details on all tasks and variables are available in the HCP documentation~\cite{van2013wu}.

We used the pre-computed 2back - 0back statistical contrast maps provided by the HCP. 
Standard preprocessing included motion correction, image registration to the MNI template (for comparison across subjects), and general linear model analysis, resulting in $91 \times 109 \times 91$ voxels.
Voxels are then resampled to $61 \times 73 \times 61$ using the \texttt{nilearn} python package\footnote{\url{http://nilearn.github.io/}} and applying standard brain masks, resulting in $65598$ voxels after masking non-grey matter regions.
($\mathbf{X}$ dataset with dimensions $497 \times 65598$).

We apply our \algo algorithm on the HCP data with arbitrarily selected parameters ${\mathrm{T}}=10^6$ and $r=5$. 
We set the target sparsity at $15\%$ for each canonical vector.
Figure~\ref{fig:brain_behavior} depicts the brain regions and the behavioral factors corresponding to the nonzero weights of the extracted canonical pair.
The map identifies a set of fronto-parietal regions known to be involved in executive function and working memory, which are the major functions isolated by the 2 back - 0-back contrast.  In addition, it identifies deactivation in the default mode areas (medial prefrontal and parietal), which is also associated with engagement of difficult cognitive functions. The behavioral variables associated with activation of this network are all related to various aspects of intelligence; the Penn Matrix Reasoning Test (\texttt{PMAT24}, a measure of fluid intelligence), picture vocabulary (\texttt{PicVocab}, a measure of language comprehension), and reading ability (\texttt{ReadEng}).

\begin{figure*}[t!]
\centerline{%
\setlength{\fboxsep}{0pt}%
\setlength{\fboxrule}{0pt}%
\begin{minipage}[td]{0.72\linewidth}
	\raisebox{-0.5\height}{%
		\fbox{%
			\includegraphics[width=1\textwidth]{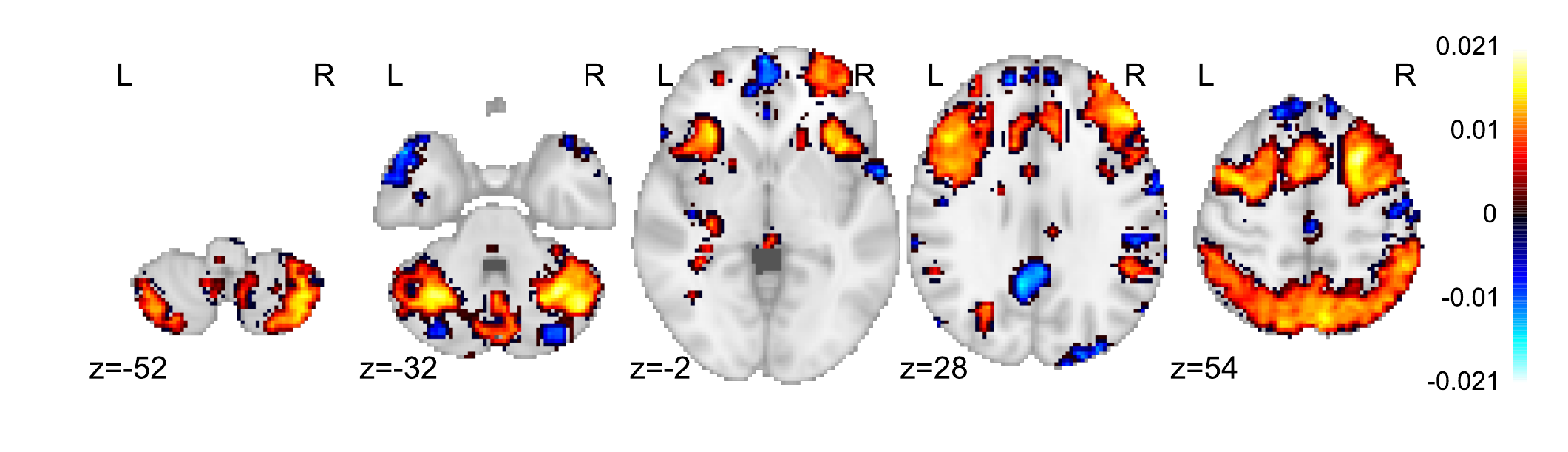}%
		}%
	}%
\end{minipage}%
\definecolor{dark-gray}{gray}{0.30}
\begin{minipage}[t]{0.28\linewidth}%
	\centering
	\raisebox{+0.1\height}{%
	\color{dark-gray}
	\fbox{\begin{tabular}{ll}%
		\toprule
		\multicolumn{2}{l}{Behavioral Factor \& Weight}\\
		\midrule
			\texttt{PMAT24\_A\_CR} 		&$0.487$\\
			\texttt{PicVocab\_AgeAdj}	&$0.448$\\
			\texttt{ReadEng\_AgeAdj} 	&$0.440$\\
			\texttt{PicVocab\_Unadj} 	&$0.433$\\
			\texttt{ReadEng\_Unadj} 	&$0.426$\\
		\bottomrule
		\end{tabular}}}
\end{minipage}%
}
\caption{%
	Brain regions and behavioral factors selected by the sparse left and right canonical vectors extracted by our \algo algorithm.
	Target sparsity is set at $15\%$ for each canonical vector
	and \algo is configured to run for ${\mathrm{T}}=10^6$ samples operating on a rank $r=5$ approximation of the input data. 
	The map identifies a set of fronto-parietal regions known to be involved in executive function and working memory
	and deactivation in the default mode areas (medial prefrontal and parietal), which is also associated with engagement of difficult cognitive functions. The behavioral variables identified to be positively correlated with the activation of this network are all related to various aspects of intelligence.
}
\label{fig:brain_behavior}
\end{figure*}

\paragraph{Parallelization}
To speed up execution,
our prototypical \texttt{Python} implementation of \algo exploits the \texttt{multiprocessing} module:
$\mathrm{N}$ independent worker processes are spawned, and each one independently performs $\mathrm{T}/\mathrm{N}$ rounds of the main iteration of Alg.~\ref{algo:cca} returning a single canonical vector pair.
The main process collects and compares the candidate pairs to determine the final output.

To demonstrate the parallelizability of our algorith, we run \algo for the aforementioned task on the brain imaging data for various values of the number $\mathrm{N}$ of workers on a single server with $36$ physical processing cores\footnote{Intel(R) Xeon(R) CPU E5-2699 v3 @ 2.30GHz} and approximately $250\mathrm{Gb}$ of main memory.
In Figure~\ref{fig:multicore_scaling} (top panel), we plot the run time with respect to the number of workers used. 
The bottom panel depicts the achieved speedup factor:
using the execution time on~$5$ worker processes as a reference value,
the speedup factor is the ratio of the execution time on $5$ processes over that on $\mathrm{N}$.
As expected, the algorithm achieved a speedup factor that grows almost linearly in the number of available processors.

\begin{figure}[t!]
\centering
\includegraphics[width=0.97\linewidth]
{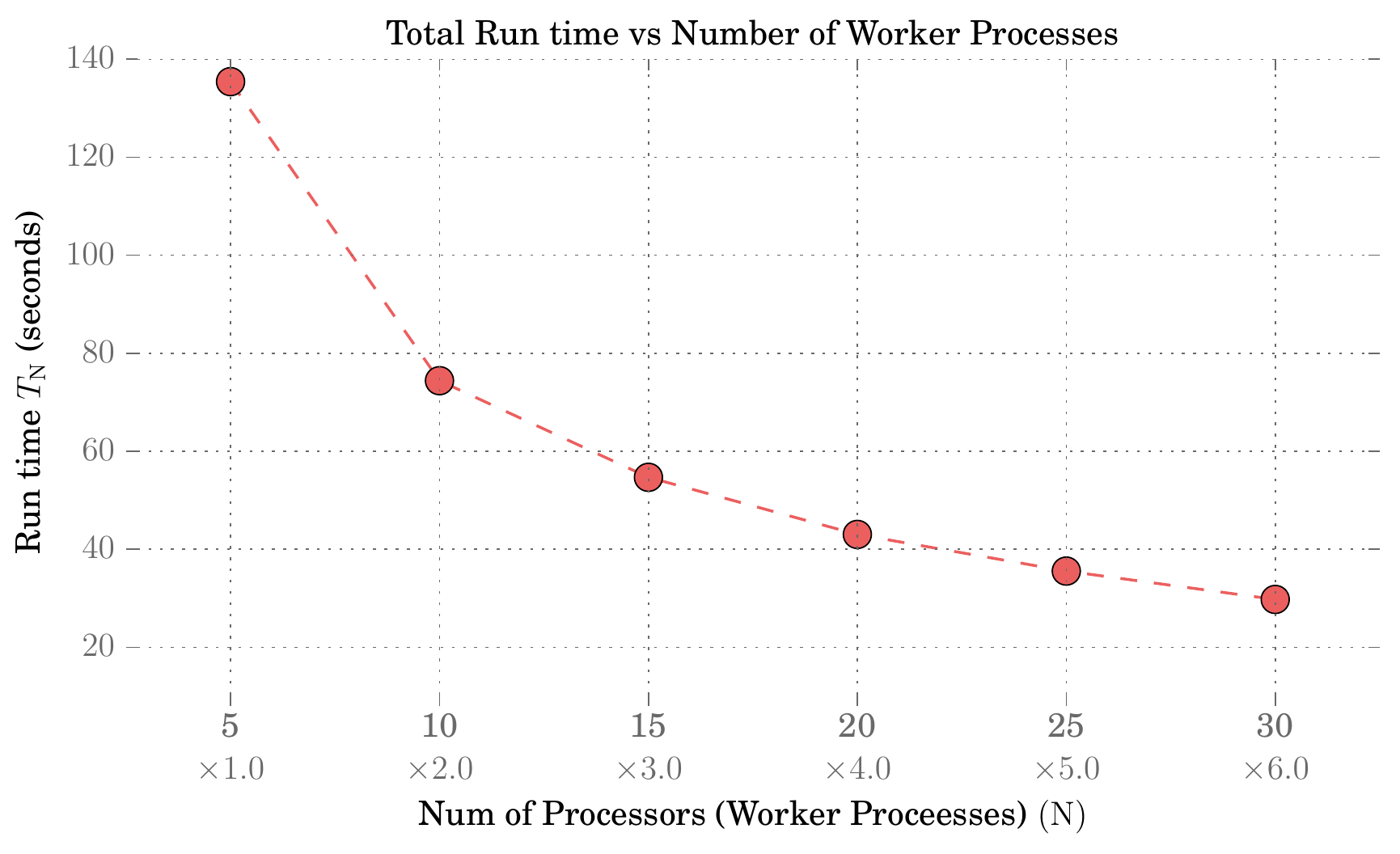}
\includegraphics[width=0.97\linewidth]{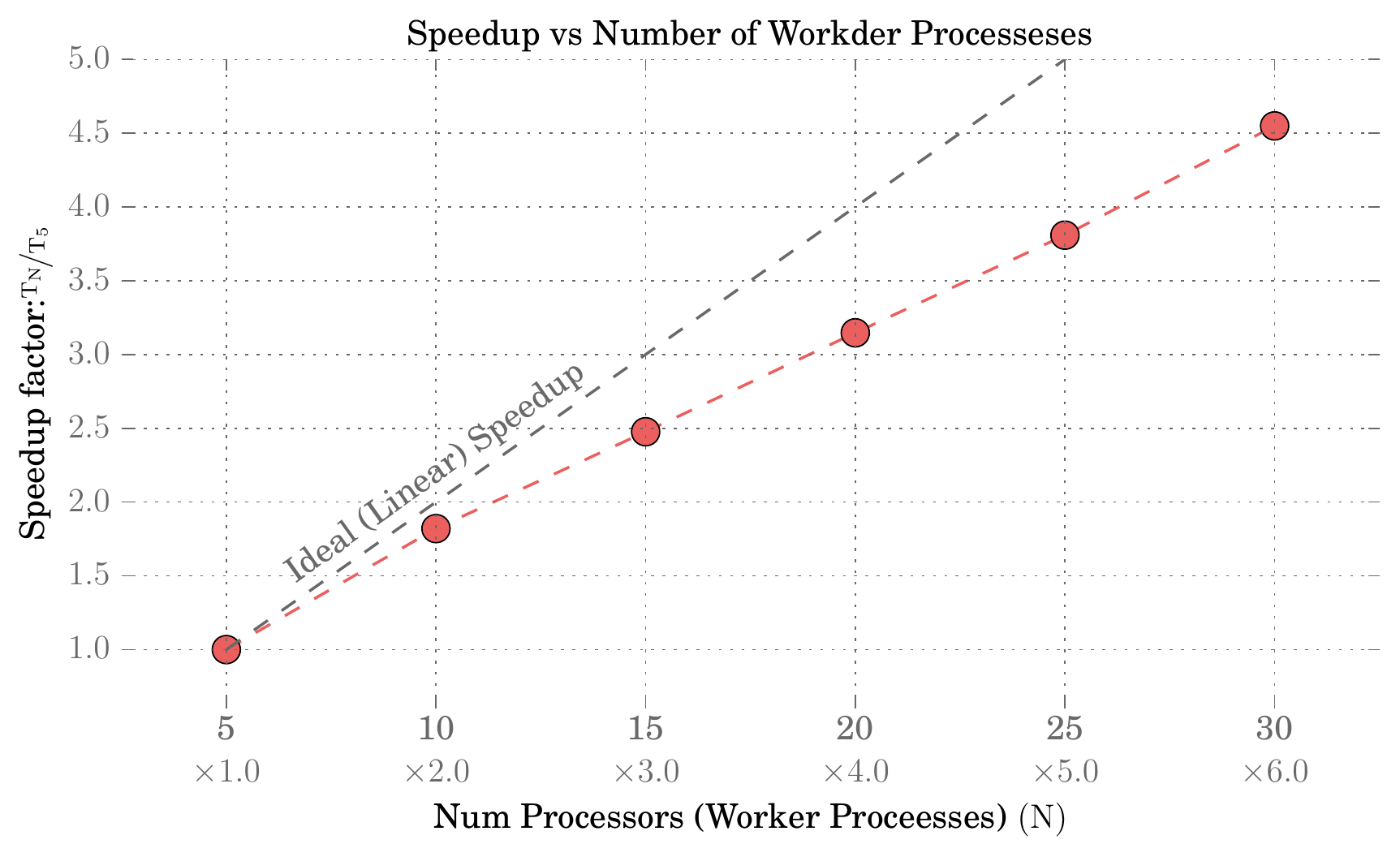}
\caption{%
Speedup factors and corresponding total execution time, achieved by the prototypical parallel implementation of SpannCCA (Alg.~\ref{algo:cca}) as a function of the number of worker processes or equivalently the number of processors used. 
Depicted values are medians over $20$ executions, each with ${T=10^5}$ and $r=5$, on the $65598 \times 38$ example discussed in section~\ref{sec:brain_imaging_dataset}.
A speedup factory approximately linear in the number of workers is achieved. 
}
\label{fig:multicore_scaling}
\end{figure}

\section{Discussion}
We presented a novel combinatorial algorithm for the sparse diagonal CCA problem and other constrained variants, with provable data-dependent global approximation guarantees and several attractive properties:
the algorithm is simple, embarrassingly parallelizable, with complexity that scales linearly in the dimension of the input data,
while it administers precise control on the sparsity of the extracted canonical vectors.
Further it can accommodate additional structural constraints by plugging in a suitable ``projection'' subroutine.

Several directions remain open. 
We addressed the question of computing a single pair of sparse canonical vectors.
Numerically, multiple pairs can be computed successively employing an appropriate deflation step.
However, determining the kind of deflation most suitable for the application at hand, 
as well as the sparsity level of each component can be a challenging task, leaving a lot of room for research.


\begin{small}
\bibliography{cca}
\bibliographystyle{icml2016}
\end{small}

\clearpage
\newpage

\section{Hardness}\label{sec:hardness}
We provide a proof for the NP-hardness of the constrained (and specifically sparse) CCA problem via a reduction from sparse PCA. 
Recall that sparse PCA is the following optimization problem:
\begin{align}
   \max_{
      \substack{
         \mathbf{x} :
         \|\mathbf{x}\|_{0} = k\\
         \|\mathbf{x}\|_{2} = 1
      }
   }
   \mathbf{x}^{\transpose}
   \mathbf{A}
   \mathbf{x},
   \label{pca:def}
\end{align}
where~$k$ is a given parameter and~$\mathbf{A}$ a given $n \times n$ positive semidefinite  (PSD) matrix.

We show that the sparse PCA problem~\eqref{pca:def} reduces to the sparse CCA problem~\eqref{cca:def} and in particular the maximization
\begin{align}
   \max_{
      \substack{
         \mathbf{x} :
         \|\mathbf{x}\|_{0} = k,
         \|\mathbf{x}\|_{2} = 1\\
         \mathbf{y} :
         \|\mathbf{y}\|_{0} = k,
         \|\mathbf{y}\|_{2} = 1
      }
   }
   \mathbf{x}^{\transpose}
   \mathbf{A}
   \mathbf{y}.
   \label{cca:reduction:sparse}
\end{align}
The only difference between~\eqref{pca:def} and~\eqref{cca:reduction:sparse} is that in the latter the optimal values for the two variables $\mathbf{x}$ and $\mathbf{y}$ may be different.
If we add the constraint $\mathbf{x}=\mathbf{y}$ in
~\eqref{cca:reduction:sparse}, then then two maximizations are identical. 
We show that this is not necessary: since $\mathbf{A}$ is PSD, the optimal solution of~\eqref{cca:reduction:sparse} will inherently satisfy $\mathbf{x}=\mathbf{y}$, and in turn the two maximizations are equivalent.

Let $\mathbf{U}, \mathbf{\Lambda}$ be the eigenvalue decomposition of~$\mathbf{A}$: the $n \times n$ matrix $\mathbf{U}$ contains the eigenvectors, while the $n \times n$ diagonal $\mathbf{\Lambda}$ contains the eigenvalues $\lambda_{1}, \hdots, \lambda_{n} \ge 0$ in decreasing order.
Let ($\mathbf{x}_{\star}$, $\mathbf{y}_{\star}$) be the optimal solution of~\eqref{cca:reduction:sparse}.
Further, let 
$\overline{\mathbf{x}} = \mathbf{U}^{\transpose}\mathbf{x}_{\star}$,
and
$\overline{\mathbf{y}} = \mathbf{U}^{\transpose}\mathbf{y}_{\star}$.
Then,
\begin{align}
   \mathbf{x}_{\star}^{\transpose}
   \mathbf{A}
   \mathbf{y}_{\star}
   =
   \mathbf{x}_{\star}^{\transpose}
   \mathbf{U}\mathbf{\Lambda}\mathbf{U}^{\transpose}
   \mathbf{y}_{\star}
   =
   \overline{\mathbf{x}}^{\transpose}
   \mathbf{\Lambda}
   \overline{\mathbf{y}}
   =
   \sum_{i=1}^{n}
      \overline{x}_{i} \overline{y}_{i} \lambda_{i}.
\end{align}

\begin{theorem}[Weighted Cauchy-Schwarz inequality; \citeappendix{cvetkovski:inequalities}, Theorem 10.1]
\label{thm:weighted_cs}
Let $a_{i}, b_{i} \in \mathbb{R}$ be real numbers and let ${m_{i} \in \mathbb{R}^{+}}$,
$i=1,2,\hdots, n$.
Then,
\begin{align}
   \left(
   \sum_{i=1}^{n}
      a_{i} b_{i} m_{i}
   \right)^{2}
   \le
   \left(
      \sum_{i=1}^{n}
         a_{i}^{2} m_{i}
      \right)
   \left(
      \sum_{i=1}^{n}
         b_{i}^{2} m_{i}
   \right).
   \nonumber
\end{align}
Equality occurs if and only if $\tfrac{a_{1}}{b_{2}} = \hdots = \tfrac{a_{n}}{b_{n}}$.
\end{theorem}
By Theorem~\ref{thm:weighted_cs},
\begin{align}
   \left(
      \mathbf{x}_{\star}^{\transpose}
      \mathbf{A}
      \mathbf{y}_{\star}
   \right)^{2}
   \le
   \left(
      \sum_{i=1}^{n}
         \overline{x}_{i}^{2} \lambda_{i}
      \right)
   \left(
      \sum_{i=1}^{n}
         \overline{y}_{i}^{2} \lambda_{i}
   \right),
   \nonumber
\end{align}
with equality if and only if there exists a constant $c \in \mathbb{R}$ such that
$\overline{\mathbf{x}} = c \cdot \overline{\mathbf{y}}$,
and taking into account that both $\overline{\mathbf{x}}$ and $\overline{\mathbf{y}}$ are unit norm vectors, it follows that $\overline{\mathbf{x}} = \overline{\mathbf{y}}$.
Finally, since $\mathbf{U}$ is a full rank matrix (in fact orthonormal basis of $\mathbb{R}^{n}$), it follows that
$\mathbf{x}_{\star}=\mathbf{y}_{\star}$ must hold.

\section{Proofs}
\label{sec:approx-proof}
For the remainder of this section, we define
\begin{align}
   \left(
      \mathbf{u}_{\star}, \mathbf{v}_{\star}
   \right)
   \eqdef
   \argmax_{\mathbf{u} \in \mathcal{U}, \mathbf{v} \in \mathcal{V}} 
      \mathbf{u}^{\transpose}\mathbf{\Sigma_{\mathsmaller{\mathrm{X}\mathrm{Y}}}}\mathbf{v},
   \nonumber
\end{align}
\textit{i.e.}, $\mathbf{u}_{\star}, \mathbf{v}_{\star}$ is a feasible pair that maximizes --not necessarily uniquely-- the objective.
Here, $\mathbf{\Sigma_{\mathsmaller{\mathrm{X}\mathrm{Y}}}}$ is a given $m \times n$ real matrix, which plays the role of
Further, we assume that there exists procedures to compute the exact soluton to
$\mathsf{P}_{\mathcal{U}}(\cdot)$ and $\mathsf{P}_{\mathcal{V}}(\cdot)$ in~\eqref{lowrank:solve-for-x} and~\eqref{lowrank:solve-for-y}, running in time 
$\timet_{\mathcal{U}}$ and $\timet_{\mathcal{V}}$, respectively.
The following results can be easily adapted for the case where these procedures yield approximate solutions.

\begin{lemma}
\label{core-rank-r-guarantees}
For any real ${m \times n}$ matrix~$\mathbf{\Sigma_{\mathsmaller{\mathrm{X}\mathrm{Y}}}}$ with $\text{rank}(\mathbf{\Sigma_{\mathsmaller{\mathrm{X}\mathrm{Y}}}})= r \le \max\lbrace m, n\rbrace$
and 
$\epsilon \in (0, 1)$,
Algorithm~\ref{algo:cca} with input $\mathbf{\Sigma_{\mathsmaller{\mathrm{X}\mathrm{Y}}}}$, $r$,
and $T=\widetilde{O}\bigl(2^{r \cdot \log_{2}(\sfrac{2}{\epsilon})} \bigr)$
outputs $\mathbf{u}_{\sharp} \in \mathcal{U}$
and $\mathbf{v}_{\sharp} \in \mathcal{V}$ such that
   \begin{align}
     \mathbf{u}_{\sharp}^{\transpose}
     \mathbf{\Sigma_{\mathsmaller{\mathrm{X}\mathrm{Y}}}}
     \mathbf{v}_{\sharp}
     \ge
     \mathbf{u}_{\star}^{\transpose} \mathbf{\Sigma_{\mathsmaller{\mathrm{X}\mathrm{Y}}}} \mathbf{v}_{\star}
     - 
     \epsilon \cdot \sigma_{1}(\mathbf{\Sigma_{\mathsmaller{\mathrm{X}\mathrm{Y}}}}),
     \nonumber
   \end{align}
   in time 
   $
   \timet_{\mathsmaller{\mathsf{SVD}}}(r) 
   + O\mathopen{}\bigl(
      T \cdot \bigl(\timet_{\mathcal{U}}+\timet_{\mathcal{V}}+r \cdot \max\lbrace m, n\rbrace \bigr)\bigr)
   $.
\end{lemma}
\begin{proof}
   In the sequel, $\mathbf{U}$, $\mathbf{\Sigma}$ and $\mathbf{V}$ are used to denote the $r$-truncated singular value decomposition of $\mathbf{\Sigma_{\mathsmaller{\mathrm{X}\mathrm{Y}}}}$.
   Note that the lemma assumes that the accuracy parameter $r$ is equal to the rank of the input matrix $\mathbf{\Sigma_{\mathsmaller{\mathrm{X}\mathrm{Y}}}}$ and hence ${\mathbf{\Sigma_{\mathsmaller{\mathrm{X}\mathrm{Y}}}} = \mathbf{U}\mathbf{\Sigma}\mathbf{V}^{\transpose}}$.

   Recall that $\mathbf{u}_{\star}, \mathbf{v}_{\star}$ is a pair ---not necessarily unique--- that maximizes the objective $\mathbf{u}^{\transpose}\mathbf{\Sigma_{\mathsmaller{\mathrm{X}\mathrm{Y}}}}\mathbf{v}$ over all feasible solutions.
   Define $\mathbf{c}_{\star} \eqdef \mathbf{V}^{\transpose}\mathbf{v}_{\star}$.
   Note that
   $\mathbf{c}_{\star}$ is a vector in $\mathbf{R}^{r \times 1}$ with $\|\mathbf{c}_{\star}\|_{2} \le 1$
   since the $r$ columns of $\mathbf{V}$ are orthonormal and ${\|\mathbf{v}_{\star}\|_{2} = 1}$.
   Finally, let $\overline{\mathbf{c}}_{\star} \eqdef \mathbf{c}_{\star} /\|\mathbf{c}_{\star}\|_{2}$.
   Note that $\|\mathbf{c}_{\star}\|_{2} > 0$ since by assumption $\mathbf{u}_{\star}^{\transpose}\mathbf{\Sigma_{\mathsmaller{\mathrm{X}\mathrm{Y}}}}\mathbf{v}_{\star} > 0$.

   Algorithm~\ref{algo:cca} operates in an iterative fashion. In each iteration, it independently considers a point~$\mathbf{c}$ selected randomly and uniformly from the $r$-dimensional $\ell_{2}$-unit sphere $\mathbb{S}_{2}^{r-1}$ and generates a candidate solution pair at each point.
   For $\mathrm{T}=\widetilde{O}\bigl(2^{r \cdot \log_{2}(\sfrac{2}{\epsilon})} \bigr)$,
   the collection of randomly sampled poins forms an $\sfrac{\epsilon}{2}$-net for $\mathbb{S}_{2}^{r-1}$.
   By definition, the $\sfrac{\epsilon}{2}$-net contains a point $\widetilde{\mathbf{c}} \in \mathbb{R}^{r \times 1}$, such that
   \begin{align}\label{eq:net_property}
      \|\widetilde{\mathbf{c}} - \overline{\mathbf{c}}_{\star}\|_{2} \le \epsilon/2.
   \end{align}
      Let $(\widetilde{\mathbf{u}}, \widetilde{\mathbf{v}})$ be the candidate solution pair computed at $\widetilde{\mathbf{c}}$ by the two step maximization procedure, \textit{i.e.}, let
   \begin{align}
      \widetilde{\mathbf{u}}
      \;\eqdef\;
      \argmax_{\mathbf{u}\in \mathcal{U}}
         \mathbf{u}^{\transpose}
         \mathbf{U}\mathbf{\Sigma}
         \widetilde{\mathbf{c}}
\end{align}
and
\begin{align}
      \widetilde{\mathbf{v}}
      \;\eqdef\;
      \argmax_{\mathbf{v}\in \mathcal{V}}
         \widetilde{\mathbf{u}}^{\transpose}
         \mathbf{\Sigma_{\mathsmaller{\mathrm{X}\mathrm{Y}}}}
         \mathbf{v}.
      \label{def-of-nearXandY}
   \end{align}
   

   By the definition of $\mathbf{c}_{\star}$, 
   and letting $\rho\eqdef \|\mathbf{c}_{\star}\|_{2}$
   \begin{align}
      {\mathbf{u}}_{\star}^{\transpose}
      \mathbf{\Sigma_{\mathsmaller{\mathrm{X}\mathrm{Y}}}}
      {\mathbf{v}}_{\star}
      &=
      \mathbf{u}_{\star}^{\transpose}
      \mathbf{U}\mathbf{\Sigma}
      \mathbf{c}_{\star}
      \nonumber\\&=
         \rho \cdot 
         \mathbf{u}_{\star}^{\transpose}
         \mathbf{U}\mathbf{\Sigma}
         \overline{\mathbf{c}}_{\star}
      \nonumber\\&=
      \rho \cdot 
         \mathbf{u}_{\star}^{\transpose}
         \mathbf{U}\mathbf{\Sigma}
         \widetilde{\mathbf{c}}
      +
         \rho \cdot 
         \mathbf{u}_{\star}^{\transpose}
         \mathbf{U}\mathbf{\Sigma}
         \bigl(\overline{\mathbf{c}}_{\star}-\widetilde{\mathbf{c}}\bigr)
      \nonumber\\&\le
         \rho \cdot 
         \widetilde{\mathbf{u}}^{\transpose}
         \mathbf{U}\mathbf{\Sigma}
         \widetilde{\mathbf{c}}
      +
         \rho \cdot 
         \mathbf{u}_{\star}^{\transpose}
         \mathbf{U}\mathbf{\Sigma}
         \bigl(\overline{\mathbf{c}}_{\star}-\widetilde{\mathbf{c}}\bigr)
      \nonumber\\&\le
         \rho \cdot 
         \widetilde{\mathbf{u}}^{\transpose}
         \mathbf{U}\mathbf{\Sigma}
         \widetilde{\mathbf{c}}
      +
         \tfrac{\epsilon}{2} \cdot \sigma_{1}(\mathbf{\Sigma_{\mathsmaller{\mathrm{X}\mathrm{Y}}}})
      .
      \label{base-inequality}
   \end{align}
   The first inequality follows from the fact that $\widetilde{\mathbf{u}}$ by definition maximizes the first term over all $\mathbf{u}\in \mathcal{U}$.
   The last inequality is due to Lemma~\ref{lemma:abs-trace-XAY-ub} and the fact that $\|\mathbf{u}_{\star}\|_{2}=1$ and $\rho \le 1$.
   We further upper bound the right hand side of~\eqref{base-inequality} as follows:
   \begin{align}
      &  \rho \cdot 
         \widetilde{\mathbf{u}}^{\transpose}
         \mathbf{U}\mathbf{\Sigma}
         \widetilde{\mathbf{c}}
      \nonumber\\&=
         \rho \cdot 
         \widetilde{\mathbf{u}}^{\transpose}
         \mathbf{U}\mathbf{\Sigma}
         \overline{\mathbf{c}}_{\star}
         +
         \rho \cdot 
         \widetilde{\mathbf{u}}^{\transpose}
         \mathbf{U}\mathbf{\Sigma}
         \left(\widetilde{\mathbf{c}} - \overline{\mathbf{c}}_{\star}\right)
      \nonumber\\&=
         \widetilde{\mathbf{u}}^{\transpose}
         \mathbf{U}\mathbf{\Sigma}
         \mathbf{c}_{\star}
         +
         \rho \cdot 
         \widetilde{\mathbf{u}}^{\transpose}
         \mathbf{U}\mathbf{\Sigma}
         \left(\widetilde{\mathbf{c}} - \overline{\mathbf{c}}_{\star}\right)
      \nonumber\\&=
         \widetilde{\mathbf{u}}^{\transpose}
         \mathbf{U}\mathbf{\Sigma}
         \mathbf{V}^{\transpose}
         \mathbf{v}_{\star}
         +
         \rho \cdot 
         \widetilde{\mathbf{u}}^{\transpose}
         \mathbf{U}\mathbf{\Sigma}
         \left(\widetilde{\mathbf{c}} - \overline{\mathbf{c}}_{\star}\right)
      \nonumber\\&\le
         \widetilde{\mathbf{u}}^{\transpose}
         \mathbf{U}\mathbf{\Sigma}
         \mathbf{V}^{\transpose}
         \widetilde{\mathbf{v}}
         +
         \rho \cdot 
         \widetilde{\mathbf{u}}^{\transpose}
         \mathbf{U}\mathbf{\Sigma}
         \left(\widetilde{\mathbf{c}} - \overline{\mathbf{c}}_{\star}\right)
      \label{optimal_y_sharp}
      \\&\le
         \widetilde{\mathbf{u}}^{\transpose}
         \mathbf{\Sigma_{\mathsmaller{\mathrm{X}\mathrm{Y}}}}
         \widetilde{\mathbf{v}}
         +
         \tfrac{\epsilon}{2} \cdot \sigma_{1}(\mathbf{\Sigma_{\mathsmaller{\mathrm{X}\mathrm{Y}}}}).
         \label{bound_on_sharp_x}
   \end{align}
   Inequality~\eqref{optimal_y_sharp} follows by the fact that $\widetilde{\mathbf{v}}$ by definition~\eqref{def-of-nearXandY} maximizes the bilinear term $\mathbf{u}^{\transpose}\mathbf{\Sigma_{\mathsmaller{\mathrm{X}\mathrm{Y}}}}\mathbf{v}$ over all ${\mathbf{v} \in \mathcal{V}}$ when $\mathbf{u}=\widetilde{\mathbf{u}}$.
   The last inequality is once again due to Lemma~\ref{lemma:abs-trace-XAY-ub} and the fact that $\|\mathbf{u}_{\star}\|_{2}=1$ and $\rho \le 1$.
   Combining~\eqref{base-inequality} and~\eqref{bound_on_sharp_x}, we obtain
   \begin{align}
         \widetilde{\mathbf{u}}^{\transpose}
         \mathbf{\Sigma_{\mathsmaller{\mathrm{X}\mathrm{Y}}}}
         \widetilde{\mathbf{v}}
      \ge
         \mathbf{u}_{\star}^{\transpose}
         \mathbf{\Sigma_{\mathsmaller{\mathrm{X}\mathrm{Y}}}}
         \mathbf{v}_{\star}
      -
      \epsilon
      \cdot \sigma_{1}(\mathbf{\Sigma_{\mathsmaller{\mathrm{X}\mathrm{Y}}}}).
      \nonumber
   \end{align}
   Algorithm~\ref{algo:cca} computes multiple candidate solution pairs and outputs the one that maximizes the objective.
   Therefore, the output pair $(\mathbf{u}_{\sharp}, \mathbf{v}_{\sharp})$ must achieve a value as least as high as that achieved by 
   $(\widetilde{\mathbf{u}}, \widetilde{\mathbf{v}})$,
   which implies the desired guarantee. 

   The running time of Algorithm~\ref{algo:cca} follows straightforwardly by inspection.
   The algorithm first computes the truncated singular value decomposition of inner dimension $r$ in time denoted by $\timet_{\texttt{SVD}}(r)$.
   Subsequently, it performs $T$ iterations.
   The cost of each iteration is determined by the cost of the matrix-vector multiplications and the running times $\timet_{\mathcal{U}}$ and $\timet_{\mathcal{V}}$ of the operators $\mathsf{P}_{\mathcal{U}}(\cdot)$ and $\mathsf{P}_{\mathcal{V}}(\cdot)$. 
   Note that matrix multiplications can exploit the available singular value decomposition of $\mathbf{\Sigma_{\mathsmaller{\mathrm{X}\mathrm{Y}}}}$ and are performed in time ${r \cdot \max\lbrace m, n\rbrace}$.
   Substituting the value of $\mathrm{T}$, completes the proof.
\end{proof}
\begin{reptheorem}{thm:main-algo-guarantees}
For any real ${m \times n}$ matrix~$\mathbf{\Sigma_{\mathsmaller{\mathrm{X}\mathrm{Y}}}}$,
$\epsilon \in (0, 1)$, and $r \le \max\lbrace m, n\rbrace$,
Algorithm~\ref{algo:cca} with input $\mathbf{\Sigma_{\mathsmaller{\mathrm{X}\mathrm{Y}}}}$, $r$,
and $T=\widetilde{O}\bigl(2^{r \cdot \log_{2}(\sfrac{2}{\epsilon})} \bigr)$
outputs $\mathbf{u}_{\sharp} \in \mathcal{U}$
and $\mathbf{v}_{\sharp} \in \mathcal{V}$ such that
   \begin{align}
     \mathbf{u}_{\sharp}^{\transpose}
     \mathbf{\Sigma_{\mathsmaller{\mathrm{X}\mathrm{Y}}}}
     \mathbf{v}_{\sharp}
     \ge
     \mathbf{u}_{\star}^{\transpose} \mathbf{\Sigma_{\mathsmaller{\mathrm{X}\mathrm{Y}}}} \mathbf{v}_{\star}
     - 
     \epsilon \cdot \sigma_{1}(\mathbf{\Sigma_{\mathsmaller{\mathrm{X}\mathrm{Y}}}})
     -
     2 \cdot \sigma_{r+1}(\mathbf{\Sigma_{\mathsmaller{\mathrm{X}\mathrm{Y}}}}),
     \nonumber
   \end{align}
   in time 
   $
   \timet_{\mathsmaller{\mathsf{SVD}}}(r) + O\mathopen{}\bigl(T \cdot \bigl(\timet_{\mathcal{U}}+\timet_{\mathcal{V}}+r \cdot \max\lbrace m, n\rbrace \bigr)\bigr)
   $.
\end{reptheorem}

\begin{proof}
Recall that Algorithm~\ref{algo:cca} with input an $m \times n$ matrix $\mathbf{\Sigma_{\mathsmaller{\mathrm{X}\mathrm{Y}}}}$ and accuracy parameter $r$,
first computes a rank-$r$ truncated singular value decomposition $\mathbf{U}$, $\mathbf{\Sigma}$, $\mathbf{V}$ and operates on that principal subspace of $\mathbf{\Sigma_{\mathsmaller{\mathrm{X}\mathrm{Y}}}}$.
Let $\mathbf{B}$ be the $m \times n$ best rank-$r$ approximation of $\mathbf{\Sigma_{\mathsmaller{\mathrm{X}\mathrm{Y}}}}$ under the spectral norm.
Then
$\mathbf{B}=\mathbf{U}\mathbf{\Sigma}\mathbf{V}^{\transpose}$.
One can easily verify that running Algorithm~\ref{algo:cca} with input $\mathbf{\Sigma_{\mathsmaller{\mathrm{X}\mathrm{Y}}}}$ and accuracy parameter $r$, is equivalent to applying the algorithm on $\mathbf{B}$ with the same parameters.

By Lemma~\ref{core-rank-r-guarantees},
Algorithm~\ref{algo:cca} outputs
$\mathbf{u}_{\sharp}, \mathbf{v}_{\sharp}$
such that
\begin{align}
     \mathbf{u}_{\sharp}^{\transpose}
     \mathbf{B}
     \mathbf{v}_{\sharp}
     \ge
     \widehat{\mathbf{u}}_{\star}^{\transpose}
     \mathbf{B}
     \widehat{\mathbf{v}}_{\star}
     - 
     \epsilon \cdot \sigma_{1}(\mathbf{B}),
     \label{eq:guar-lowrank}
\end{align}
where 
\begin{align}
   \left(
      \widehat{\mathbf{u}}_{\star}, 
      \widehat{\mathbf{v}}_{\star}
   \right)
   \eqdef
   \argmax_{\mathbf{u} \in \mathcal{U}, \mathbf{v} \in \mathcal{V}} 
      \mathbf{u}^{\transpose}\mathbf{B}\mathbf{v}
   \nonumber
\end{align}
is a pair that optimally solves the maximization on the rank-$r$ matrix $\mathbf{B}$.
By the optimality of the pair
$\widehat{\mathbf{u}}_{\star}, 
\widehat{\mathbf{v}}_{\star}$ for the rank-$r$ problem,
it follows that
\begin{align}
   \widehat{\mathbf{u}}_{\star}^{\transpose}
   \mathbf{B}
   \widehat{\mathbf{v}}_{\star}
   \ge
   \mathbf{u}_{\star}^{\transpose}
   \mathbf{B}
   \mathbf{v}_{\star}.
   \label{eq:optimal_low_better}
\end{align}
Recall that $\mathbf{u}_{\star}$, $\mathbf{v}_{\star}$ is the pair that optimally solves the maximization on the original input matrix~$\mathbf{\Sigma_{\mathsmaller{\mathrm{X}\mathrm{Y}}}}$.
Further,
\begin{align}
   \mathbf{u}_{\star}^{\transpose}
   \mathbf{B}
   \mathbf{v}_{\star}
   &=
   \mathbf{u}_{\star}^{\transpose}
   \mathbf{\Sigma_{\mathsmaller{\mathrm{X}\mathrm{Y}}}}
   \mathbf{v}_{\star}
   -
   \mathbf{u}_{\star}^{\transpose}
   (\mathbf{\Sigma_{\mathsmaller{\mathrm{X}\mathrm{Y}}}}-\mathbf{B})
   \mathbf{v}_{\star}
   \nonumber\\& \ge
   \mathbf{u}_{\star}^{\transpose}
   \mathbf{\Sigma_{\mathsmaller{\mathrm{X}\mathrm{Y}}}}
   \mathbf{v}_{\star}
   -
   \bigl\lvert
   \mathbf{u}_{\star}^{\transpose}
   (\mathbf{\Sigma_{\mathsmaller{\mathrm{X}\mathrm{Y}}}}-\mathbf{B})
   \mathbf{v}_{\star}
   \bigr\rvert
   \nonumber\\& \ge
   \mathbf{u}_{\star}^{\transpose}
   \mathbf{\Sigma_{\mathsmaller{\mathrm{X}\mathrm{Y}}}}
   \mathbf{v}_{\star}
   -
   \sigma_{r+1}(\mathbf{\Sigma_{\mathsmaller{\mathrm{X}\mathrm{Y}}}}).
   \label{eq:bp03}
\end{align}
Combining~\eqref{eq:bp03} with
\eqref{eq:guar-lowrank} and 
\eqref{eq:optimal_low_better},
\begin{align}
   \mathbf{u}_{\sharp}^{\transpose}
   \mathbf{B}
   \mathbf{v}_{\sharp}
   \ge
   \mathbf{u}_{\star}^{\transpose}
   \mathbf{\Sigma_{\mathsmaller{\mathrm{X}\mathrm{Y}}}}
   \mathbf{v}_{\star}
   -
   \sigma_{r+1}(\mathbf{\Sigma_{\mathsmaller{\mathrm{X}\mathrm{Y}}}})
   - 
   \epsilon \cdot \sigma_{1}(\mathbf{B}).
   \nonumber
\end{align}
Finally, 
\begin{align}
   \mathbf{u}_{\sharp}^{\transpose}
   \mathbf{B}
   \mathbf{v}_{\sharp}
&=
   \mathbf{u}_{\sharp}^{\transpose}
   \mathbf{\Sigma_{\mathsmaller{\mathrm{X}\mathrm{Y}}}}
   \mathbf{v}_{\sharp}
   -
   \mathbf{u}_{\sharp}^{\transpose}
   (\mathbf{\Sigma_{\mathsmaller{\mathrm{X}\mathrm{Y}}}}-\mathbf{B})
   \mathbf{v}_{\sharp}
   \nonumber\\ 
&\le
   \mathbf{u}_{\sharp}^{\transpose}
   \mathbf{\Sigma_{\mathsmaller{\mathrm{X}\mathrm{Y}}}}
   \mathbf{v}_{\sharp}
   +
   \bigl\lvert
      \mathbf{u}_{\sharp}^{\transpose}
      (\mathbf{\Sigma_{\mathsmaller{\mathrm{X}\mathrm{Y}}}}-\mathbf{B})
      \mathbf{v}_{\sharp}
   \bigr\rvert.
   \nonumber\\ 
&\le
   \mathbf{u}_{\sharp}^{\transpose}
   \mathbf{\Sigma_{\mathsmaller{\mathrm{X}\mathrm{Y}}}}
   \mathbf{v}_{\sharp}
   +
   \sigma_{r+1}(\mathbf{\Sigma_{\mathsmaller{\mathrm{X}\mathrm{Y}}}}).
   \label{eq:bp02}
\end{align}
Combining with the previous inequality,
we obtain
\begin{align}
   \mathbf{u}_{\sharp}^{\transpose}
   \mathbf{\Sigma_{\mathsmaller{\mathrm{X}\mathrm{Y}}}}
   \mathbf{v}_{\sharp}
   \ge
   \mathbf{u}_{\star}^{\transpose}
   \mathbf{\Sigma_{\mathsmaller{\mathrm{X}\mathrm{Y}}}}
   \mathbf{v}_{\star}
   -
   2 \cdot \sigma_{r+1}(\mathbf{\Sigma_{\mathsmaller{\mathrm{X}\mathrm{Y}}}})
   - 
   \epsilon \cdot \sigma_{1}(\mathbf{B}).
   \nonumber
\end{align}
Noting that ${\sigma_{1}(\mathbf{B}) = \sigma_{1}(\mathbf{\Sigma_{\mathsmaller{\mathrm{X}\mathrm{Y}}}})}$ completes the proof of the approximation guarantee. 
The running time of the algorithm is established in Lemma~\ref{core-rank-r-guarantees}.
\end{proof}

\begin{lemma}
\label{same-lemma-for-the-special-case}
For any real ${m \times n}$ matrix~$\mathbf{\Sigma_{\mathsmaller{\mathrm{X}\mathrm{Y}}}}$ with $\text{rank}(\mathbf{\Sigma_{\mathsmaller{\mathrm{X}\mathrm{Y}}}})= r \le \max\lbrace m, n\rbrace$
and 
$\epsilon \in (0, 1)$,
if $\mathcal{V} = \left\{\mathbf{v}~:~\|\mathbf{v}\|_2 = 1\right\}$,
then
Algorithm~\ref{algo:cca} with input $\mathbf{\Sigma_{\mathsmaller{\mathrm{X}\mathrm{Y}}}}$, $r$,
and $\mathrm{T}=\widetilde{O}\bigl(2^{r \cdot \log_{2}(\sfrac{2}{\epsilon})} \bigr)$
outputs $\mathbf{u}_{\sharp} \in \mathcal{U}$
and $\mathbf{v}_{\sharp} \in \mathcal{V}$ such that
   \begin{align}
     \mathbf{u}_{\sharp}^{\transpose}
     \mathbf{\Sigma_{\mathsmaller{\mathrm{X}\mathrm{Y}}}}
     \mathbf{v}_{\sharp}
     \ge
     (1-\epsilon) \cdot \mathbf{u}_{\star}^{\transpose} \mathbf{\Sigma_{\mathsmaller{\mathrm{X}\mathrm{Y}}}} \mathbf{v}_{\star}
     \nonumber
   \end{align}
   in time 
   $
   \timet_{\mathsmaller{\mathsf{SVD}}}(r) + O\mathopen{}\bigl(\mathrm{T} \cdot \bigl(\timet_{\mathcal{U}}+\timet_{\mathcal{V}}+r \cdot \max\lbrace m, n\rbrace \bigr)\bigr)
   $.
\end{lemma}

\begin{proof}
The lemma focuses on the special case where the feasible region for the variable $\mathbf{v}$ coincides with the set of vectors with unit $\ell_{2}$ norm,
\textit{i.e.}, 
\begin{align}
   \mathcal{V}
   =
   \lbrace
   \mathbf{v}: \|\mathbf{v}\|_{2}=1
   \rbrace.
   \label{Y_unit_norm}
\end{align}
The feasible region $\mathcal{U}$ for $\mathbf{u}$ is arbitrary, assuming once again that there exists an efficient operator $\mathsf{P}_{\mathcal{U}}(\cdot)$.

By the Cauchy-Schwarz inequality,
for any $\mathbf{u}_{0} \in \mathbb{R}^{m \times 1}$,
\begin{align}
   \mathbf{u}_{0}^{\transpose}\mathbf{\Sigma_{\mathsmaller{\mathrm{X}\mathrm{Y}}}}\mathbf{v}
   &\le
   \|\mathbf{u}_{0}^{\transpose}\mathbf{\Sigma_{\mathsmaller{\mathrm{X}\mathrm{Y}}}}\|_{2},
   \quad \forall \mathbf{v} \in \mathcal{V}.
\end{align}
In fact, equality is achieved when $\mathbf{v}$ is aligned with 
$\mathbf{\Sigma_{\mathsmaller{\mathrm{X}\mathrm{Y}}}}^{\transpose}\mathbf{u}_{0}$,
\textit{i.e.}, for 
$\mathbf{v} = \mathbf{\Sigma_{\mathsmaller{\mathrm{X}\mathrm{Y}}}}^{\transpose}\mathbf{u}_{0}/\|\mathbf{\Sigma_{\mathsmaller{\mathrm{X}\mathrm{Y}}}}^{\transpose}\mathbf{u}_{0}\|_{2} \in \mathcal{V}$.
In turn, for any $\mathbf{u}_{0} \in \mathbb{R}^{m \times 1}$,
\begin{align}
   \max_{\mathbf{v} \in \mathcal{V}}
   \mathbf{u}_{0}^{\transpose}\mathbf{\Sigma_{\mathsmaller{\mathrm{X}\mathrm{Y}}}}\mathbf{v}
   &=
   \mathbf{u}_{0}^{\transpose}\mathbf{\Sigma_{\mathsmaller{\mathrm{X}\mathrm{Y}}}}
   \mathbf{\Sigma_{\mathsmaller{\mathrm{X}\mathrm{Y}}}}^{\transpose}\mathbf{u}_{0}/\|\mathbf{\Sigma_{\mathsmaller{\mathrm{X}\mathrm{Y}}}}^{\transpose}\mathbf{u}_{0}\|_{2}
   \nonumber\\
   &=
   \|\mathbf{\Sigma_{\mathsmaller{\mathrm{X}\mathrm{Y}}}}^{\transpose}\mathbf{u}_{0}\|_{2}
   \nonumber\\
   &=
   \|\mathbf{V}\mathbf{\Sigma}\mathbf{U}^{\transpose}\mathbf{u}_{0}\|_{2}
   \nonumber\\
   &=
   \|\mathbf{\Sigma}\mathbf{U}^{\transpose}\mathbf{u}_{0}\|_{2},
   \label{eq:y_objective_as_x_norm}
\end{align}
where the last equality follows from the fact that the $r$ columns of $\mathbf{V}$ are orthonormal.

We now proceed in a fashion very similar to that in the proof of Lemma~\ref{core-rank-r-guarantees}.
Recall that $\mathbf{u}_{\star}, \mathbf{v}_{\star}$ is a pair that maximizes --not necessarily uniquely-- the objective $\mathbf{u}^{\transpose}\mathbf{\Sigma_{\mathsmaller{\mathrm{X}\mathrm{Y}}}}\mathbf{v}$ over all feasible solutions, 
and define $\mathbf{c}_{\star} \eqdef \mathbf{V}^{\transpose}\mathbf{v}_{\star}$.
Note that here,
\begin{align}
   \mathbf{c}_{\star}
   \eqdef
   \mathbf{V}^{\transpose}\mathbf{v}_{\star}
   &=
   \mathbf{V}^{\transpose}\mathbf{\Sigma_{\mathsmaller{\mathrm{X}\mathrm{Y}}}}^{\transpose}\mathbf{u}_{\star}/\|\mathbf{\Sigma_{\mathsmaller{\mathrm{X}\mathrm{Y}}}}^{\transpose}\mathbf{u}_{\star}\|_{2}
   \nonumber\\&=
   \mathbf{\Sigma}\mathbf{U}^{\transpose}\mathbf{u}_{\star}/\|\mathbf{V}\mathbf{\Sigma}\mathbf{U}^{\transpose}\mathbf{u}_{\star}\|_{2}
   \nonumber\\&=
   \mathbf{\Sigma}\mathbf{U}^{\transpose}\mathbf{u}_{\star}/\|\mathbf{\Sigma}\mathbf{U}^{\transpose}\mathbf{u}_{\star}\|_{2}
\end{align}
and hence, $\|\mathbf{c}_{\star}\|_{2}=1$.
Following similar reasoning as in the proof of Lemma~\ref{core-rank-r-guarantees},
Algorithm~\ref{algo:cca} considers a point $\widetilde{\mathbf{c}} \in \mathbb{R}^{r \times 1}$, such that
\begin{align}
   \|\widetilde{\mathbf{c}} - \overline{\mathbf{c}}_{\star}\|_{2} \le \epsilon.
   \nonumber
\end{align}
Let $(\widetilde{\mathbf{u}}, \widetilde{\mathbf{v}})$ be the candidate solution pair computed at $\widetilde{\mathbf{c}}$ by the two step maximization procedure.
We have,
\begin{align}
   {\mathbf{u}}_{\star}^{\transpose}
   \mathbf{\Sigma_{\mathsmaller{\mathrm{X}\mathrm{Y}}}}
   {\mathbf{v}}_{\star}
   &=
   \mathbf{u}_{\star}^{\transpose}
   \mathbf{U}\mathbf{\Sigma}
   \mathbf{c}_{\star}
   \nonumber\\&=
      \mathbf{u}_{\star}^{\transpose}
      \mathbf{U}\mathbf{\Sigma}
      \widetilde{\mathbf{c}}
   + 
      \mathbf{u}_{\star}^{\transpose}
      \mathbf{U}\mathbf{\Sigma}
      \bigl(\mathbf{c}_{\star}-\widetilde{\mathbf{c}}\bigr)
   \nonumber\\&\le
      \widetilde{\mathbf{u}}^{\transpose}
      \mathbf{U}\mathbf{\Sigma}
      \widetilde{\mathbf{c}}
   +
      \|\mathbf{u}_{\star}^{\transpose}
      \mathbf{U}\mathbf{\Sigma}\|_{2}
      \|\mathbf{c}_{\star}-\widetilde{\mathbf{c}}\|_{2}
   \nonumber\\&\le
      \widetilde{\mathbf{u}}^{\transpose}
      \mathbf{U}\mathbf{\Sigma}
      \widetilde{\mathbf{c}}
   +
      \epsilon \cdot \|\mathbf{u}_{\star}^{\transpose}
      \mathbf{U}\mathbf{\Sigma}\|_{2}
   \label{eq:special-case-upper-bound}
   .
\end{align} 
where the first inequality follows from the fact that $\widetilde{\mathbf{u}}$ by definition maximizes the first term at $\widetilde{\mathbf{c}}$ over all $\mathbf{u} \in \mathcal{U}$ and the Cauchy-Schwarz inequality.
The key difference from the proof of Lemma~\ref{core-rank-r-guarantees}, 
is that the term $\|\mathbf{u}_{\star}^{\transpose}\mathbf{U}\mathbf{\Sigma}\|_{2}$ in the right-hand side coincides with the optimal objective value $ {\mathbf{u}}_{\star}^{\transpose}\mathbf{\Sigma_{\mathsmaller{\mathrm{X}\mathrm{Y}}}}{\mathbf{v}}_{\star}$ as follows from~\eqref{eq:y_objective_as_x_norm}. 
For comparison, note that in the proof of Lemma~\ref{core-rank-r-guarantees} it was loosely upper bounded by~${\sigma_{1}(\mathbf{\Sigma_{\mathsmaller{\mathrm{X}\mathrm{Y}}}})}$.
Continuing from~\eqref{eq:special-case-upper-bound},
\begin{align}
   \left(1-\epsilon\right)
   \cdot
   {\mathbf{u}}_{\star}^{\transpose}
   \mathbf{\Sigma_{\mathsmaller{\mathrm{X}\mathrm{Y}}}}
   {\mathbf{v}}_{\star}
&\le
      \widetilde{\mathbf{u}}^{\transpose}
      \mathbf{U}\mathbf{\Sigma}
      \widetilde{\mathbf{c}}
   .
   \label{eq:almost_there}
\end{align}   
But, once again by the Cauchy-Schwarz inequality,
\begin{align}
   \widetilde{\mathbf{u}}^{\transpose}
   \mathbf{U}\mathbf{\Sigma}
   \widetilde{\mathbf{c}}
   &\le
   \|\widetilde{\mathbf{u}}^{\transpose}
   \mathbf{U}\mathbf{\Sigma}\|_{2}
   \|\widetilde{\mathbf{c}}\|_{2}
   \nonumber\\&=
   \|\widetilde{\mathbf{u}}^{\transpose}
   \mathbf{U}\mathbf{\Sigma}\|_{2}
\end{align}
and by~\eqref{eq:y_objective_as_x_norm},
\begin{align}
   \widetilde{\mathbf{u}}^{\transpose}
   \mathbf{U}\mathbf{\Sigma}
   \widetilde{\mathbf{c}}
   &=
   \max_{\mathbf{v} \in \mathcal{V}} 
      \widetilde{\mathbf{u}}^{\transpose}\mathbf{\Sigma_{\mathsmaller{\mathrm{X}\mathrm{Y}}}}\mathbf{v}
   =
      \widetilde{\mathbf{u}}^{\transpose}\mathbf{\Sigma_{\mathsmaller{\mathrm{X}\mathrm{Y}}}}\widetilde{\mathbf{v}}
   .
   \label{eq:last-cookie}
\end{align}   
Combining~\eqref{eq:last-cookie} with~\eqref{eq:almost_there},
\begin{align}
   \widetilde{\mathbf{u}}^{\transpose}\mathbf{\Sigma_{\mathsmaller{\mathrm{X}\mathrm{Y}}}}\widetilde{\mathbf{v}}
   \ge
   (1-\epsilon) \cdot 
   \mathbf{u}_{\star}^{\transpose}\mathbf{\Sigma_{\mathsmaller{\mathrm{X}\mathrm{Y}}}}\mathbf{v}_{\star}
\end{align}
Recalling that Algorithm~\ref{algo:cca} outputs the candidate pair that maximizes the objective among all computed feasible points implies the desired result.
\end{proof}

\begin{reptheorem}{thm:main-algo-guarantees-special-case}
For any real ${m \times n}$ matrix~$\mathbf{\Sigma_{\mathsmaller{\mathrm{X}\mathrm{Y}}}}$ 
and 
$\epsilon \in (0, 1)$,
if $\mathcal{V} = \left\{\mathbf{v}~:~\|\mathbf{v}\|_2 = 1\right\}$,
then
Algorithm~\ref{algo:cca} with input $\mathbf{\Sigma_{\mathsmaller{\mathrm{X}\mathrm{Y}}}}$, $r$,
and $T=\widetilde{O}\bigl(2^{r \cdot \log_{2}(\sfrac{2}{\epsilon})} \bigr)$
outputs $\mathbf{u}_{\sharp} \in \mathcal{U}$
and $\mathbf{v}_{\sharp} \in \mathcal{V}$ such that
   \begin{align}
     \mathbf{u}_{\sharp}^{\transpose}
     \mathbf{\Sigma_{\mathsmaller{\mathrm{X}\mathrm{Y}}}}
     \mathbf{v}_{\sharp}
     \ge
     (1-\epsilon) \cdot \mathbf{u}_{\star}^{\transpose} \mathbf{\Sigma_{\mathsmaller{\mathrm{X}\mathrm{Y}}}} \mathbf{v}_{\star}
     -
     2 \cdot \sigma_{r+1}(\mathbf{\Sigma_{\mathsmaller{\mathrm{X}\mathrm{Y}}}})
     \nonumber
   \end{align}
   in time 
   $
   \timet_{\mathsmaller{\mathsf{SVD}}}(r) + O\mathopen{}\bigl(T \cdot \bigl(\timet_{\mathcal{U}}+\timet_{\mathcal{V}}+r \cdot \max\lbrace m, n\rbrace \bigr)\bigr)
   $.
\end{reptheorem}

\begin{proof}
  The Theorem follows from Lemma~\ref{same-lemma-for-the-special-case}.
  The proof is similar to that of Theorem~\ref{thm:main-algo-guarantees}.  
  The main difference lies in substituting \eqref{eq:guar-lowrank} with 
  \begin{align}
     \mathbf{u}_{\sharp}^{\transpose}
     \mathbf{B}
     \mathbf{v}_{\sharp}
     \ge
     (1-\epsilon) \cdot 
     \mathbf{u}_{\mathsmaller{(B)}\star}^{\transpose}
     \mathbf{B}
     \mathbf{v}_{\mathsmaller{(B)}\star}.
     \label{Y_unit:low-rank-guarantee}
\end{align}
The remainder of the proof easily follows.
\end{proof}


\section{Auxiliary Lemmata}
\begin{lemma}
   \label{holder-consequence}
   Let $a_1,\hdots , a_n$ 
   and $b_1, \hdots, b_n$ be $2n$ real numbers and let $p$ and
   $q$ be two numbers such that ${1/p} + {1/q} = 1$ and $p>1$. We have
   \begin{align}
	  \bigl\lvert 
		 \sum_{i=1}^{n} a_{i}b_{i}
	  \bigr\rvert
	  \le
	  \bigl( \sum_{i=1}^{n} \lvert  a_{i}\rvert^{p} \bigr)^{1/p}
	  \cdot
	  \bigl( \sum_{i=1}^{n} \lvert  b_{i}\rvert^{q} \bigr)^{1/q}.
	  \nonumber
   \end{align}
\end{lemma}
\begin{lemma}
   \label{lemma:inner-prod-ub}
   For any $\mathbf{A}, \mathbf{B} \in \mathbb{R}^{n \times k}$,
   \begin{align}
	  \bigl\lvert 
		 \langle \mathbf{A}, \mathbf{B} \rangle
	  \bigr\rvert 
	  \eqdef 
	  \bigl\lvert
		 \trace\bigl( \mathbf{A}^{\transpose} \mathbf{B}\bigr)
	  \bigr\rvert
	  \le
	  \|\mathbf{A}\|_{\frob}
	  \|\mathbf{B}\|_{\frob}.
	  \nonumber
   \end{align}
\end{lemma}
\begin{proof}
 Treating $\mathbf{A}$ and $\mathbf{B}$ as vectors, 
 the lemma follows immediately from Lemma~\ref{holder-consequence} for $p=q=2$.
\end{proof}
\begin{lemma}
   \label{lemma:frob-of-matrix-prod}
   For any two real matrices $\mathbf{A}$ and $\mathbf{B}$ of appropriate dimensions,
   \begin{align}
	  \|\mathbf{A}\mathbf{B}\|_{\frob}
	  \le
	  \min\mathopen{}
	  \bigl\lbrace
		 \|\mathbf{A}\|_{2} \|\mathbf{B}\|_{\frob}, \;
		 \|\mathbf{A}\|_{\frob} \|\mathbf{B}\|_{2}
	  \bigr\rbrace.
	  \nonumber
   \end{align}
\end{lemma}
\begin{proof}
   Let $\mathbf{b}_{i}$ denote the $i$th column of $\mathbf{B}$. 
   Then,
   \begin{align}
	  \|\mathbf{A}\mathbf{B}\|_{\frob}^{2}
	 & =
	  \sum_{i} \| \mathbf{A} \mathbf{b}_{i} \|_{2}^{2}
	  \le
	  \sum_{i} \| \mathbf{A}\|_{2}^{2}  \|\mathbf{b}_{i} \|_{2}^{2}
	  \nonumber\\&=
	  \| \mathbf{A}\|_{2}^{2} \sum_{i}   \|\mathbf{b}_{i} \|_{2}^{2}
	  =
	  \| \mathbf{A}\|_{2}^{2} \|\mathbf{B} \|_{\frob}^{2}.
	  \nonumber 
   \end{align}
   Similarly, using the previous inequality,
   \begin{align}
	  \|\mathbf{A}\mathbf{B}\|_{\frob}^{2}
	  =
	  \|\mathbf{B}^{\transpose}\mathbf{A}^{\transpose}\|_{\frob}^{2}
	  \le
	  \|\mathbf{B}^{\transpose}\|_{2}^{2}\|\mathbf{A}^{\transpose}\|_{\frob}^{2}
	  =
	  \|\mathbf{B}\|_{2}^{2}\|\mathbf{A}\|_{\frob}^{2}.
	  \nonumber
   \end{align}
   The desired result follows combining the two upper bounds.
\end{proof}
\begin{lemma}
   \label{lemma:abs-trace-XAY-ub}
   For any real 
   $m \times k$ matrix $\mathbf{X}$, $m \times n$ matrix $\mathbf{A}$, and  $n \times k$ matrix $\mathbf{Y}$,
   \begin{align}
	  \bigl\lvert
	  \trace\bigl( \mathbf{X}^{\transpose} \mathbf{A} \mathbf{Y}\bigr)
	  \bigr\rvert
	  \le
	  \|\mathbf{X}\|_{\frob} \cdot \|\mathbf{A} \|_{2} \cdot \|\mathbf{Y}\|_{\frob}.
	  \nonumber
   \end{align}
\end{lemma}
\begin{proof}
   We have
   \begin{align}
	  \bigl\lvert
	  \trace\bigl( \mathbf{X}^{\transpose} \mathbf{A} \mathbf{Y}\bigr)
	  \bigr\rvert
	  \le
	  \|\mathbf{X}\|_{\frob} \cdot \|\mathbf{A}\mathbf{Y}\|_{\frob}
	  \le
	  \|\mathbf{X}\|_{\frob} \cdot \|\mathbf{A} \|_{2} \cdot \|\mathbf{Y}\|_{\frob},
	  \nonumber
   \end{align}
   with the first inequality following from Lemma~\ref{lemma:inner-prod-ub}
   on $\lvert \langle \mathbf{X},\, \mathbf{A}\mathbf{Y}\rangle\rvert$ and the second from Lemma~\ref{lemma:frob-of-matrix-prod}.
\end{proof}
\begin{lemma}
   \label{lemma:abs-trace-XAY-ub-orthogonal}
   For any real
   $m \times n$ matrix $\mathbf{A}$,
   and pair of 
   $m \times k$ matrix $\mathbf{X}$ and  $n \times k$ matrix $\mathbf{Y}$
   such that $\mathbf{X}^{\transpose}\mathbf{X}=\mathbf{I}_{k}$ and $\mathbf{Y}^{\transpose}\mathbf{Y}=\mathbf{I}_{k}$
   with $k \le \min\lbrace {m},\; {n}\rbrace$, the following holds:
   \begin{align}
	  \bigl\lvert
	  \trace\bigl( \mathbf{X}^{\transpose} \mathbf{A} \mathbf{Y}\bigr)
	  \bigr\rvert
	  \le
	  \sqrt{k}\cdot 
	  \bigl( \sum_{i=1}^{k} \sigma_{i}^{2}\bigl(\mathbf{A}\bigr) \bigr)^{1/2}.
	  \nonumber
   \end{align}
\end{lemma}
\begin{proof}
   By Lemma~\ref{lemma:inner-prod-ub},
   \begin{align}
	  \lvert \langle \mathbf{X},\, \mathbf{A}\mathbf{Y}\rangle\rvert
	  &=
	  \bigl\lvert
	  \trace\bigl( \mathbf{X}^{\transpose} \mathbf{A} \mathbf{Y}\bigr)
	  \bigr\rvert
	  \nonumber\\&
	  \le
	  \|\mathbf{X}\|_{\frob} \cdot \|\mathbf{A}\mathbf{Y}\|_{\frob}
	  =
	  \sqrt{k} \cdot \|\mathbf{A}\mathbf{Y}\|_{\frob}.
	  \nonumber
   \end{align}
   where the last inequality follows from the fact that $\|\mathbf{X}\|_{\frob}^{2} = \trace\bigl(\mathbf{X}^{\transpose}\mathbf{X}\bigr) = \trace\bigl(\mathbf{I}_{k}\bigr) = k$.
   Further, for any $\mathbf{Y}$ such that $\mathbf{Y}^{T}\mathbf{Y}= \mathbf{I}_{k}$,
   \begin{align}
   		\|\mathbf{A} \mathbf{Y}\|_{\frob}^{2}
		\le
   		\max_{
			\substack{
				\widehat{\mathbf{Y}} \in \mathbb{R}^{n \times k}\\
				\widehat{\mathbf{Y}}^{\transpose}\widehat{\mathbf{Y}} = \mathbf{I}_{k} 
			}
		}
   		\|\mathbf{A} \widehat{\mathbf{Y}} \|_{\frob}^{2}
		=
		\sum_{i=1}^{k} \sigma_{i}^{2}(\mathbf{A}).
   \end{align}
   Combining the two inequalities, the result follows. 
\end{proof}
\newpage
\begin{lemma}
   For any real
   $m \times n$ matrix $\mathbf{A}$,
   and any $k \le \min\lbrace {m},\; {n}\rbrace$, 
      \begin{align}
   		\max_{
			\substack{
				{\mathbf{Y}} \in \mathbb{R}^{n \times k}\\
				{\mathbf{Y}}^{\transpose}{\mathbf{Y}} = \mathbf{I}_{k} 
			}
		}
   		\|\mathbf{A} {\mathbf{Y}} \|_{\frob}
		=
		\left(\sum_{i=1}^{k} \sigma_{i}^{2}(\mathbf{A}) \right)^{1/2}.
		\nonumber
   \end{align}
   The above equality is realized when the $k$ columns of~$\mathbf{Y}$ coincide with the $k$ leading right singular vectors of~$\mathbf{A}$.
\end{lemma}
\begin{proof}
	Let $\mathbf{U}\mathbf{\Sigma}\mathbf{V}^{\transpose}$ be the singular value decomposition of $\mathbf{A}$;
	$\mathbf{U}$ and $\mathbf{V}$ are $m \times m$ and $n \times n$ unitary matrices respectively, while $\Sigma$ is a diagonal matrix with $\Sigma_{jj} = \sigma_{j}$, the $j$th largest singular value of $\mathbf{A}$, $j=1, \hdots, d$, where $d \eqdef \min\lbrace {m}, {n} \rbrace$.
	Due to the invariance of the Frobenius norm under unitary multiplication,
   \begin{align}
   		\|\mathbf{A} \mathbf{Y}\|_{\frob}^{2}
		=
		\|\mathbf{U}\mathbf{\Sigma}\mathbf{V}^{\transpose} \mathbf{Y}\|_{\frob}^{2}
		=
		\|\mathbf{\Sigma}\mathbf{V}^{\transpose} \mathbf{Y}\|_{\frob}^{2}.
		\label{frob-norm-unitary}
   \end{align}
   Continuing from~\eqref{frob-norm-unitary},
   \begin{align}
		\|\mathbf{\Sigma}\mathbf{V}^{\transpose} \mathbf{Y}\|_{\frob}^{2}
		&=
		\trace\bigl(\mathbf{Y}^{\transpose}\mathbf{V}\mathbf{\Sigma}^{2}\mathbf{V}^{\transpose} \mathbf{Y}\bigr)
		\nonumber\\
		&=
		\sum_{i=1}^{k} \mathbf{y}_{i}^{\transpose} 
			\left( 
				\sum_{j=1}^{d} \sigma_{j}^{2} \cdot \mathbf{v}_{j} \mathbf{v}_{j}^{\transpose}
			\right)
			\mathbf{y}_{i}
		\nonumber\\
		&=
		\sum_{j=1}^{d} 
			\sigma_{j}^{2} \cdot \sum_{i=1}^{k}
			\left( \mathbf{v}_{j}^{\transpose} \mathbf{y}_{i}\right)^{2}.
		\nonumber
   \end{align}
   Let $z_{j} \eqdef \sum_{i=1}^{k} \bigl( \mathbf{v}_{j}^{\transpose} \mathbf{y}_{i}\bigr)^{2}$, $j=1, \hdots, d$.
   Note that each individual $z_{j}$ satisfies
   \begin{align}
   		 0 \le z_{j} \eqdef \sum_{i=1}^{k} \bigl( \mathbf{v}_{j}^{\transpose} \mathbf{y}_{i}\bigr)^{2}
		\le
		\|\mathbf{v}_{j}\|^{2} = 1,
		\nonumber
   \end{align}
   where the last inequality follows from the fact that the columns of $\mathbf{Y}$ are orthonormal.
   Further,
   \begin{align}
   		\sum_{j=1}^{d} z_{j} 
		&=  
		\sum_{j=1}^{d}\sum_{i=1}^{k} \bigl( \mathbf{v}_{j}^{\transpose} \mathbf{y}_{i}\bigr)^{2}
		=
		\sum_{i=1}^{k}\sum_{j=1}^{d} \bigl( \mathbf{v}_{j}^{\transpose} \mathbf{y}_{i}\bigr)^{2}
		\nonumber\\&=
		\sum_{i=1}^{k}\|\mathbf{y}_{i}\|^{2} = k.
		\nonumber
   \end{align}
   Combining the above, we conclude that
   \begin{align}
		\|\mathbf{A}\mathbf{Y}\|_{\frob}^{2}
		=
		\sum_{j=1}^{d} 
			\sigma_{j}^{2} \cdot z_{j}
		\le \sigma_{1}^{2} + \hdots + \sigma_{k}^{2}.
		\label{frob-prod-ub}
   \end{align}
   Finally, it is straightforward to verify that if $\mathbf{y}_{i} = \mathbf{v}_{i}$, $i=1, \hdots, k$, then~\eqref{frob-prod-ub} holds with equality.
\end{proof}

\begin{small}
\bibliographyappendix{cca}
\bibliographystyleappendix{icml2016}
\end{small}

\end{document}